\documentclass[letterpaper, 10 pt, conference]{ieeeconf}

\IEEEoverridecommandlockouts                              

\usepackage{cite}
\usepackage{amsmath,amssymb,amsfonts}
\usepackage{mathtools}
\usepackage{multirow}
\usepackage{booktabs}
\usepackage{graphicx}
\usepackage[table]{xcolor}
\usepackage{textcomp}
\usepackage{subcaption}
\usepackage{algorithm} 
\usepackage{algpseudocode}
\usepackage{graphicx}
\usepackage{xcolor}
\usepackage{amsmath}
\usepackage[colorlinks=true]{hyperref} 
\usepackage{tabularx} 
\newcolumntype{L}{>{\raggedright\arraybackslash}X}
\usepackage[export]{adjustbox} 
\usepackage{hyperref}

\usepackage[font=small,labelfont=bf]{caption}

\newif\ifhideauthors
\hideauthorsfalse

\newif\ifhideappendix
\hideappendixfalse

\newcommand{\authorswitch}[2]{%
  \ifhideauthors #2\else #1\fi
}

\newcommand{\appendixswitch}[2]{%
  \ifhideappendix #2\else #1\fi
}

\DeclareMathOperator*{\argmin}{arg\,min}
\long\def\anonymize#1{}

\title{\LARGE \bf
Safety-aware Model Predictive Path Integral Control with Signal Temporal Logic}

\authorswitch{\author{Yiqi Zhao$^{1}$, Taekyung Kim$^{2}$, Hideki Okamoto$^{3}$, Bardh Hoxha$^{3}$, Jyotirmoy V. Deshmukh$^{1}$, Lars Lindemann$^{4}$, \\Georgios Fainekos$^{3}$
\thanks{$^{1}$Yiqi Zhao and Jyotirmoy V. Deshmukh are with the Thomas Lord Department of Computer Science, University of Southern California, Los Angeles, CA 90089, USA.
        {\tt\small \{yiqizhao, jdeshmuk\}@usc.edu}}%
\thanks{$^{2}$Taekyung Kim is with the Department of Robotics, University of Michigan, Ann Arbor, MI 48109, USA.
        {\tt\small taekyung@umich.edu}}%
\thanks{$^{3}$ Hideki Okamoto, Bardh Hoxha, and Georgios Fainekos are with the Toyota Motor North America R\&D, Ann Arbor, MI 48105, USA. {\tt\small \{hideki.okamoto, bardh.hoxha, georgios.fainekos\}@toyota.com}}
\thanks{$^{4}$Lars Lindemann is with the Automatic Control Laboratory, ETH Z\"urich, Z\"urich, Switzerland. {\tt\small llindemann@ethz.ch}}
\thanks{This work was primarily carried out during the summer internships of Yiqi Zhao and Taekyung Kim at Toyota Motor North America R\&D.}
}}{}

\begin{document}

\maketitle
\begin{figure*}
    \centering    
    \includegraphics[width=0.63\textwidth]{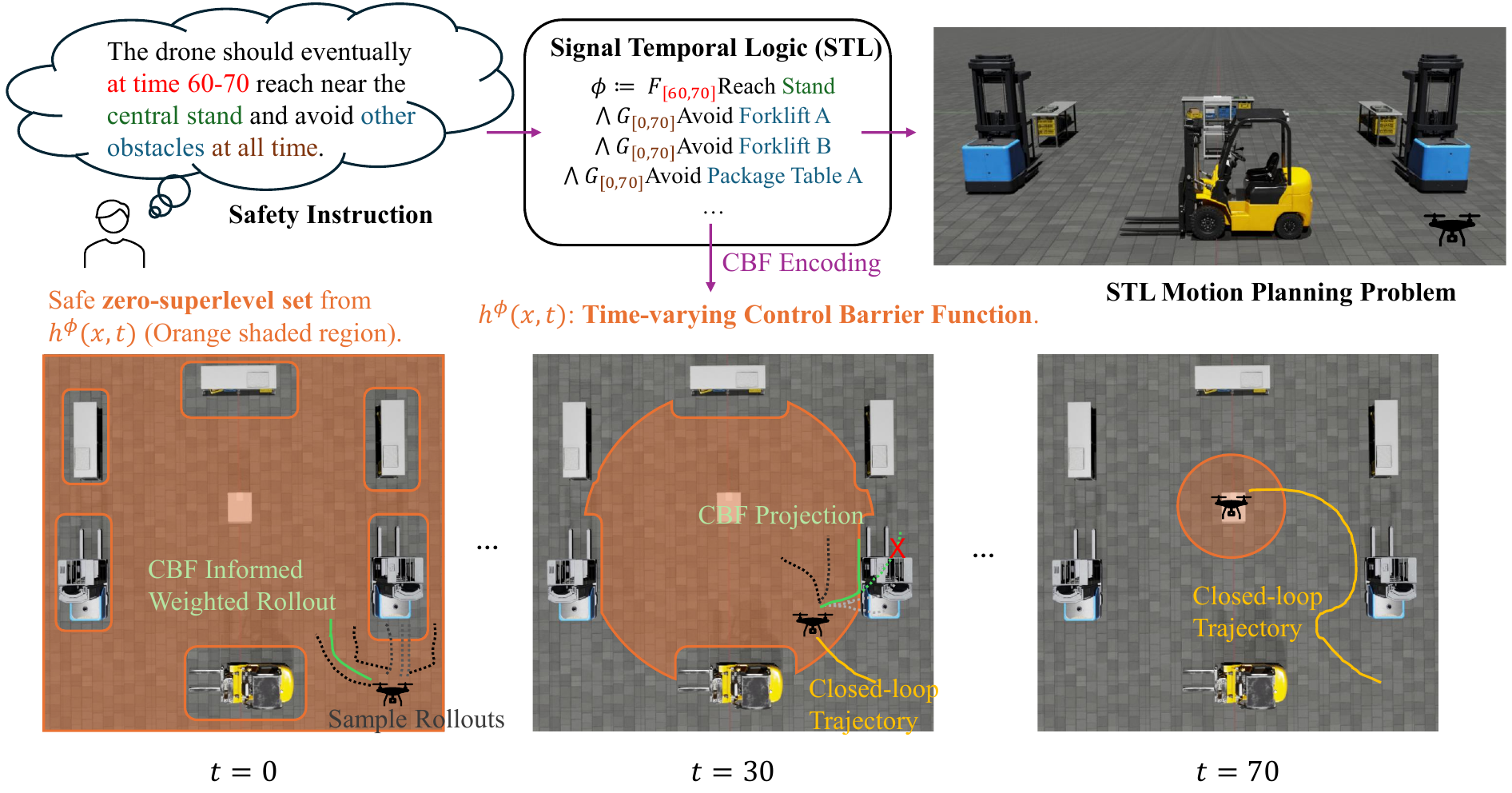}
    \caption{Safety-aware-stl-mppi: The top row highlights the STL motion planning problem with an STL formula $\phi$. Our solution is on the second row, where we rely on a time-varying CBF $h^\phi$, which encodes the region consistent with the STL constraint with its time-varying zero-super level set (shown in orange). We plan with MPPI using CBF informed weighted rollout, and project any weighted rollout (dashed green) violating the CBF into the safety region (solid green). The planning works in receding-horizon with a closed-loop trajectory (in yellow) satisfying $\phi$. The drawings in the figure are hypothetical. We show drawings in a case study in Section \ref{sec:case_study}.}
    \label{fig:overview}
    \vspace{-15pt}
\end{figure*}
\thispagestyle{empty}
\pagestyle{empty}

\newtheorem{definition}{Definition} 
\newtheorem{theorem}{Theorem} 
\newtheorem{assumption}{Assumption} 
\newtheorem{problem}{Problem} 
\newtheorem{remark}{Remark}
\newtheorem{lemma}{Lemma}
\newtheorem{corollary}{Corollary}
\newtheorem{proposition}{Proposition}
\newtheorem{example}{Example}

\begin{abstract}
Safety-aware motion planning remains a challenge in robotics, especially when missions are time-critical and are under complex specifications. In this paper, we propose safety-aware-stl-mppi, a computationally efficient sampling-based receding-horizon planning framework designed to promote satisfaction of constraints expressed in Signal Temporal Logic (STL). Our approach encodes discrete-time STL formulas into candidate time-varying control barrier functions (CBF), which are integrated into a model predictive path integral (MPPI) controller. Our method inherits the benefits of low computational cost from an efficiently parallelizable sampling based planner and utilizes CBF for constraints expressed in STL. We compare against several MPPI baselines using four artificial Mars Rover planning case studies with a diverse environment and cost setups, where we show our method consistently achieving high safety and efficiency. We show a quadcopter planning experiment with NVIDIA Isaac Lab.
\end{abstract}

\section{Introduction}
Signal temporal logic (STL) \cite{maler2004monitoring, fainekos2009robustness} arises as a commonly-used formal language that can express a rich set of real-time spatio-temporal safety and mission level requirements. For instance, STL has been used as a specification language for the responsibility sensitive safety model \cite{hekmatnejad2019encoding} for autonomous vehicles. Planning in STL with rigorous safety guarantees has been studied in \cite{raman2014model}, however with a high computational cost from solving mixed integer programs. Realizing this, \cite{lindemann2018control} proposes control barrier functions for STL specifications over continuous-time systems, which lead to convex programs given quadratic cost functions. Encodings for discrete-time stochastic systems have been studied in \cite{kordabad2024control}. In real-life applications, however, cost functions are generally nonconvex. Realizing the computational difficulty in such planning problems, sampling-based planners such as Model Predictive Path Integral are studied for general motion planning.

Model predictive path integral (MPPI) has been proposed \cite{williams2016aggressive, williams2018information} to solve optimal control problems by sampling the control inputs from a normal distribution and executing a weighted average over the rollouts based on performances. MPPI is easily parallelizable and has optimality guarantees \cite{homburger2025optimality}. The works in \cite{baldini2024don, halder2025trajectory, parwana2024model} took the first step in using STL specifications as the cost function in MPPI design. Ensuring satisfaction of general STL formulas remains unresolved.

\emph{In this paper, we view safety as satisfaction of a user-defined STL formula.} We propose \textbf{safety-aware-stl-mppi} (Figure \ref{fig:overview}), a safety-aware sampling-based receding-horizon control strategy under STL specification with general cost. Particularly, given a specification in STL, we first construct a time-varying control barrier function adapting the approach of \cite{lindemann2018control}. We incorporate the CBF in weighting the samples from our planner and design an efficient projection of the MPPI weighted rollout onto the CBF zero-superlevel set to promote STL satisfaction. We summarize our contributions:
\begin{itemize}
    \item We propose safety-aware-stl-mppi, a sampling-based receding-horizon planner that incorporates STL specifications as CBF-informed constraints while minimizing a user-defined cost function.
    \item We adapt the time-varying control barrier function encoding for STL formulas from \cite{lindemann2018control} to discrete-time and design a lightweight projection toward satisfaction of the associated CBF constraints.
    \item We present numeric case studies and a quadcopter simulation in Section \ref{sec:case_study}. The empirical results show that our algorithm achieves high safety, not attained with other MPPI baselines. We additionally demonstrate the computational efficiency of our approach. \authorswitch{We release codes associated with the algorithm at \href{https://zhaoy37.github.io/safety-aware-stl-mppi/}{https://zhaoy37.github.io/safety-aware-stl-mppi/}.}{}
\end{itemize}

\subsection{Related Literature}

\textbf{Planning with Temporal Logic.} Temporal logics are extensions of propositional and predicate logics to describe timed system behaviors. Linear temporal logic (LTL) \cite{pnueli1977temporal} and STL \cite{maler2004monitoring, fainekos2009robustness} are rich specification languages widely used for planning. Abstraction-based temporal logic planning \cite{fainekos2009temporal, ding2014optimal} constructs discrete abstraction over systems and uses graph-based planning. Optimization-based approaches use mixed integer programs \cite{raman2014model}, which are expensive to run online. Control barrier functions are used to encode STL specifications \cite{lindemann2018control}, leading to convex programs when the cost functions are restricted to be convex. Computationally efficient sampling-based methods are proposed in \cite{vasile2017sampling, kantaros2020stylus}. Recently, diffusion models \cite{zhong2022guided}, vision-language-action model based approaches \cite{zhao2026logic}, and neural predictive models are used in \cite{meng2023signal} for STL-based trajectory generation and planning but without safety guarantee. Different from these, we focus on addressing satisfaction in STL on an MPPI planner with customized cost functions that can differ from the STL constraints.

\textbf{Model Predictive Path Integral.} Model predictive path integral (MPPI) was proposed in \cite{williams2016aggressive, williams2018information} as a sampling-based receding-horizon controller. Reach-avoid specifications are considered in \cite{parwana2024model, parwana2025br} where time-invariant barrier functions are considered to ensure safety. Path-integral based offline policy learning methods in continuous-time are proposed in \cite{dimos2019learning, varnai2019prescribed, varnai2020guided} for STL tasks but are not designed to be optimized in closed-loop and guarantee satisfaction. Differentiable STL robustness is used as the objective of a Stein variational path-integral optimizer for long-horizon open-loop trajectory synthesis \cite{zheng2026stl}. The work closest to ours are \cite{baldini2024don, halder2025trajectory, halder2026lexicographic, han2026signal, bouzid2026autonomous}, which augment STL satisfaction as part of costs in MPPI rollouts. Existing cost-only MPPI methods do not provide a direct connection to a formal sufficient condition for STL satisfiaction. Our method is grounded in an exact STL-CBF condition, while its nominal real-time implementation approximates this condition. We introduce the baselines in detail in Section \ref{sec:case_study}. 

\section{Preliminaries}
Consider a control system of the form $x_{t + 1} = f(x_t, u_t) \coloneq f_0(x_t)+ g(x_t)u_t$ with $t \in \mathbb{Z}_{\ge 0}$, where $x_t \in \mathbb{R}^n$ and $u_t \in \mathcal{U} \subseteq \mathbb{R}^m$ are the state and input with actuation limit $\mathcal{U}$, assumed to be box constraints (i.e., the constraints are in the form $e_i^Tu_t \le u_{\max, i}$ or $-e_i^Tu_t \le -u_{\min, i}, \text{ where } i \in \{1, \hdots, m\}$, $e_i$ is the $i$-th Euclidean basis vector, and $u_{\min, i}, u_{\max, i} \in \mathbb{R}$). Assume $f_0:\mathbb{R}^n \rightarrow \mathbb{R}^n$ and $g:\mathbb{R}^n \rightarrow \mathbb{R}^{n \times m}$ are Lipschitz continuous and differentiable.

\begin{table*}[ht]
  \centering
  \renewcommand{\arraystretch}{1.2}
  \setlength{\tabcolsep}{6pt}
  \footnotesize
  \begin{tabularx}{\textwidth}{l L}
\hline
\textbf{Formula} & \textbf{Encoding Rule (which we follow when constructing the CBF)} \\
\hline
$\pi^\mu$ & $\forall x \in D, h^{\pi^\mu}(x, \tau_0) \ge 0$ implies $(x, \tau_0) \models \pi^\mu$.\\
\hline
$G_{[a,b]}\pi^\mu$ 
& $\forall t' \in \sigma(a, b),\forall x \in D, h^{G_{[a, b]}\pi^\mu}(x, t')\ge 0 \text{ implies } \mu(x) \ge 0$.\\ 
\hline
$F_{[a, b]}\pi^\mu$ 
& $\exists t' \in \sigma(a, b),\forall x \in D, h^{F_{[a, b]}\pi^\mu}(x, t')\ge 0 \text{ implies } \mu(x) \ge 0$.\\ 
\hline
$\mu_1 U_{[a, b]} \mu_2$ 
&  $h^{\mu_1 U_{[a, b]} \mu_2}(x, t) \coloneq \widetilde{\text{min}}(h_1(x, t), h_2(x, t))$ where $ \exists t' \in \sigma(a, b), \forall x \in D, h_2(x, t') \ge 0 \text{ implies } \mu_2(x) \ge 0$ \\& and $\forall t'' \in \sigma(0, t'), \forall x \in D, h_1(x, t'') \ge 0 \text{ implies }\mu_1(x) \ge 0$.\\ 
\hline
$G_{[a,b]} \mu_1 \wedge\mu_2$ 
& $h^{G_{[a, b]}\mu_1 \wedge\mu_2}(x, t) \coloneq\widetilde{\text{min}}(h_1(x, t), h_2(x, t))$ where $\forall t' \in \sigma(a, b), \forall x \in D, h_1(x, t') \ge 0 \text{ implies } \mu_1(x) \ge 0$ \\& and $h_2(x, t') \ge 0 \text{ implies } \mu_2(x) \ge 0$. \\ 
\hline
$F_{[a,b]} \mu_1 \wedge\mu_2$ 
& $h^{F_{[a, b]}\mu_1 \wedge\mu_2}(x, t) \coloneq \widetilde{\text{min}}(h_1(x, t), h_2(x, t))$ where $\exists t' \in \sigma(a, b), \forall x \in D, h_1(x,t')  \ge 0 \text{ implies } \mu_1(x) \ge 0$ \\&and $h_2(x, t') \ge 0 \text{ implies } \mu_2(x) \ge 0$.\\ 
\hline
$\mu_1 \wedge \mu_2 U_{[a, b]}(x, t) \mu_3 \wedge \mu_4$ 
& $h^{\mu_1 \wedge \mu_2 U_{[a, b]}(x, t) \mu_3 \wedge \mu_4}(x, t) \coloneq\widetilde{\text{min}}(h_1(x, t), \hdots, h_4(x, t))$ where \\&$\exists t' \in \sigma(a, b), \forall x \in D, h_3(x, t') \ge 0 \text{ implies } \mu_3(x) \ge 0$ and $h_4(x, t') \ge 0 \text{ implies } \mu_4(x) \ge 0$ and \\& $\forall t'' \in \sigma(0, t'), \forall x \in D, h_1(x, t'') \ge 0 \text{ implies } \mu_1(x) \ge 0$ and $h_2(x, t'') \ge 0 \text{ implies } \mu_2(x)$.\\ 
\hline
$\bigwedge_{i=1}^p \phi_i$
& $h^{\wedge_{i = 1}^p\phi_i}(x, t) \coloneq \widetilde{\text{min}}(h_1(x, t), \hdots, h_p(x, t))$ where $h_i(x, t)$ is the CBF for $\phi_i$ following the recursive definition.\\ 
\hline
\end{tabularx}
\caption{Candidate CBF functions for STL encoding}
  \label{tab:stl_encodings}
  \vspace{-15pt}
\end{table*}
We consider in this paper a fragment of signal temporal logic (STL) specifications following \cite{lindemann2018control}. We refer readers interested in the general syntax and semantics to \cite{maler2004monitoring, bartocci2018specification}. The STL syntax is defined recursively as follows
\begin{equation}
\label{eq:syntax}
    \begin{split}
        \varphi &\coloneq \top \mid \pi^\mu \mid \neg\pi^\mu \mid \varphi_1 \wedge \varphi_2,\\
        \phi &\coloneq G_{[a, b]}\varphi \mid F_{[a, b]} \varphi \mid \varphi_1 U_{[a, b]} \varphi_2 \mid \phi_1 \wedge \phi_2 \mid \varphi
    \end{split}
\end{equation}
where $\varphi_1$ and $\varphi_2$ are of type $\varphi$ and $\phi_1$ and $\phi_2$ are of type $\phi$. We denote the Boolean True and False with $\top$ and $\bot$ respectively. We emphasize that the atomic elements of an STL formula are the predicates $\pi^\mu:\mathbb{R}^n \rightarrow \{\top, \bot\}$, which are defined via a predicate function $\mu: \mathbb{R}^n \rightarrow \mathbb{R}$ such that $\pi(x_t) \coloneq \top$ if $\mu(x_t) \ge 0$ and $\pi(x_t) \coloneq \bot$ otherwise. Let us define $\sigma(\tau_1, \tau_2) \coloneq [\tau_1, \tau_2] \cap \{0, \hdots, H\}$, where we assume $H$ is sufficiently large to capture the finite-length specification (see \cite{sadraddini2015robust} for computing lower bounds of $H$). We say that $x$ satisfies an STL formula $\phi$ at time $\tau_0$ if $(x, \tau_0) \models \phi$ and
\begin{align*}
    (x, \tau_0) \models \pi^\mu &\Leftrightarrow \mu(x_{\tau_0}) \ge 0, \\
    (x, \tau_0) \models \neg \pi^\mu &\Leftrightarrow \neg((x, \tau_0) \models \pi^\mu),\\
    (x, \tau_0) \models \varphi_1 \wedge \varphi_2 &\Leftrightarrow (x, \tau_0) \models \varphi_1 \wedge (x, \tau_0) \models \varphi_2,\\
    (x, \tau_0) \models G_{[a, b]}\varphi &\Leftrightarrow \forall \tau \in \sigma(\tau_0 + a, \tau_0 + b),(x, \tau) \models \varphi,\\
    (x, \tau_0) \models F_{[a, b]}\varphi &\Leftrightarrow \exists \tau \in \sigma(\tau_0 + a, \tau_0 + b),(x, \tau) \models \varphi,\\
    (x, \tau_0) \models \varphi_1 U_{[a, b]} \varphi_2 &\Leftrightarrow \exists \tau_1 \in \sigma(\tau_0 + a, \tau_0 + b),(x, \tau_1) \models \varphi_2\\
    &\wedge \forall \tau_2 \in \sigma(\tau_0, \tau_1), (x, \tau_2) \models \varphi_1,\\
    (x, \tau_0) \models \phi_1 \wedge \phi_2 &\Leftrightarrow (x, \tau_0) \models \phi_1 \wedge (x, \tau_0) \models \phi_2.
\end{align*}
We assume in this paper that $a, b \ge 0$. STL has a robust semantics $\rho^\phi(x, \tau_0) \in \mathbb{R}$ which measures how robustly a formula $\phi$ is satisfied by $x$ at $\tau_0$. Importantly, $(x, \tau_0)\models \phi$ if $\rho^\phi(x, \tau_0) > 0$ \cite{fainekos2009robustness}. We refer the readers to the definition of robust semantics in \cite{donze2010robust, fainekos2009robustness}.

\section{Problem Formulation}
In receding-horizon control, we are given a user-defined planning horizon $\mathcal{H} \le H - t$. We are also given user-defined cost functions $q:\mathbb{R}^n \times \mathbb{R}^m\rightarrow \mathbb{R}$ denoting the stage cost and $E:\mathbb{R}^n \rightarrow \mathbb{R}$ denoting the terminal cost. At each current time $t \in \{0, \hdots, H - 1\}$ during planning, we are asked to find an optimal sequence of control inputs $v_t^{*, (t)}, \hdots, v_{t + \mathcal{H} - 1}^{*, (t)}$ via solving an open-loop optimal control problem (OCP)
\begin{align}
\label{eq:ocp}
    &v_t^{*, (t)}, \hdots, v_{t + \mathcal{H} - 1}^{*, (t)} \coloneq  \argmin_{u_t, \hdots, u_{t + \mathcal{H} - 1} \in \mathcal{U}}\{E(x^{(t)}_{t + \mathcal{H}}) \\&+ \sum_{\tau = t}^{t + \mathcal{H} - 1}q(x^{(t)}_\tau, u_\tau) \text{ s.t. } x^{(t)}_{\tau + 1} \coloneq f(x^{(t)}_\tau, u_\tau)\}.\nonumber
\end{align}
Note that we use the superscript $(t)$ here to denote the current time when we solve the OCP and the subscript to denote the predictive steps. After solving the OCP at time $t$, we plan in closed-loop by executing $v_t^{*, (t)}$ and increment the current time $t$ to solve the OCP again\footnote{As standard practice, $\mathcal{H}$ is shrunk when $t + \mathcal{H} > H$.}. Through this procedure, we attain $x^*$, which corresponds to the closed-loop trajectory from the receding-horizon procedure. Namely, $x^*_{t+1}\coloneq f(x^*_t, v^{*, (t)}_t)$ with the closed-loop solution $v^* \coloneq (v^{*, (0)}_0, \hdots, v^{*, (H - 1)}_{H - 1})$. Receding-horizon control allows trajectories robust against system disturbances.
\begin{problem}
\label{prob:formulation}
Given a system $f$, the initial state $x_0$, a planning horizon $\mathcal{H}$, and cost functions $q$ and $E$, design a feedback control s.t. $(x^*, 0) \models \phi$, where $\phi$ is a user-defined STL specification. At each time $t \in \{0, \hdots, H - 1\}$, we would like to optimize the OCP defined in \eqref{eq:ocp}.
\end{problem}

\section{Safety-Aware-STL-MPPI}
We propose safety-aware-stl-mppi to approach Problem \ref{prob:formulation}. The algorithm relies on MPPI \cite{williams2018information}, a sampling-based method proposed to solve OCPs without constraint satisfaction guarantee. In an attempt to enable hard constraint satisfaction, time-invariant CBFs were first considered in MPPI-CBF \cite{tao2022control}, Shield-MPPI \cite{yin2023shield}, and BR-MPPI \cite{parwana2025br} as a safety filter to enforce non-STL safety constraints. Inspired by these works, we propose to incorporate time-varying CBF in MPPI to satisfy STL constraints. We first introduce how time-varying CBF can be used to encode STL formulas, which we adapt to discrete-time from \cite{lindemann2018control, lindemann2025formal} where continuous-time CBFs are discussed. This design choice is naturally motivated by our formulation in Problem \ref{prob:formulation}.

\label{sec:alg}
\subsection{Time-varying Control Barrier Function for STL}
\label{subsec:stl_cbf}
Consider the system $f$ and a function $h:D \times \mathbb{Z}_{\ge0} \rightarrow \mathbb{R}$ where $D \subseteq \mathbb{R}^n$ denotes the domain in which the system operates. We assume $h$ to be differentiable in $x$ and consider a time-varying set $\mathcal{C}_t \coloneq \{x \in D \mid h(x, t) \ge 0\}$.
\begin{definition}
\label{def:cbf}
    \textbf{Time-varying Control Barrier Function.} $h$ is a candidate time-varying control barrier function (CBF) to $f$ if $C_t \neq \emptyset, \forall t \in \{0, \hdots, H\}$. A candidate CBF $h$ is valid if $\forall t \in \{0, \hdots, H - 1\}, \forall x \in C_t$, $\exists u \in \mathcal{U}$ such that
    \begin{align}
    \label{eq:cbf_cons}
        h(f(x, u), t + 1) - h(x, t) \ge -\alpha h(x, t),
    \end{align}
    where $\alpha \in (0, 1]$ is a positive constant.
\end{definition}

Consider $S(x, t) \coloneq \{u \in \mathcal{U} \mid \eqref{eq:cbf_cons}\}$. We present Theorem \ref{thm:forward}\appendixswitch{, whose proof can be found in Appendix \ref{proof:forward}}{.}
\begin{theorem}
\label{thm:forward}
    \textbf{Forward Invariance.} Let $x_0 \in \mathcal{C}_0$ and consider a sequence of inputs $(u_0, \hdots, u_{H - 1})$ where $u_t \in S(x_t, t)$. Then, $x_t \in C_t, \forall t \in \{0, \hdots, H\}$.
\end{theorem}

Given an STL formula $\phi$ from fragment \eqref{eq:syntax}, the goal now is to design a CBF $h^\phi(x, t)$ such that $h^\phi(x, t) \ge 0, \forall t \in \{0, \hdots, H\}$ ensures $(x, 0) \models \phi$. We follow the recursive semantics and start by considering the predicate. Given a predicate $\pi^\mu$, we consider any function $h^{\pi^\mu}$ such that $h^{\pi^\mu}(x, \tau_0) \ge 0 \implies (x, \tau_0) \models \pi^\mu$. A natural choice for $h^{\pi^\mu}$ is $\mu$ if $\mu$ is differentiable. Otherwise, we consider a smooth construction (e.g. $h^{\pi^\mu}(x,t) \coloneq a^2 - \|x\|^2_2$ with scalar $a > 0$ is used to encode $\pi^\mu$ where $\mu \coloneq a - \|x\|_2$). Since negation is only defined over predicates, $\phi$ can be assumed negation-free without loss of generality \cite{sadraddini2015robust} and thus the encoding for negation is ignored. We can then design CBFs for more complex formulas recursively following the encoding rules listed in Table \ref{tab:stl_encodings}, where the rows 2-4 encode temporal operators over predicates, rows 5-7 encode temporal operators over conjunctions, and the last row encodes conjunctions over temporal operators. We emphasize on the use of a smooth underapproximation of the piecewise minimum operator $\min_{i \in \{1, \hdots,p\}}h_i(x, t) \ge \widetilde{\text{min}}(h_1(x, t), \hdots, h_p(x, t)) \coloneq -\ln(\sum_{i = 1}^p\exp(-h_i(x, t)))$. See concrete examples of constructed CBFs in Section \ref{sec:case_study}. We emphasize that constructive synthesis procedures for continuous-time CBFs are discussed in \cite{marchesini2026sampling}.

Note that the CBF encoded following Table \ref{tab:stl_encodings} does not necessarily guarantee $\mathcal{C}_t \neq \emptyset$ for all $t \in \{0, \hdots, H\}$, required for the candidacy in Definition \ref{def:cbf}. In fact, a deletion mechanism can be considered to reduce conservatism \cite{lindemann2018control}, which we summarize here for discrete-time: For any $\phi$ with temporal operators, the general form of the constructed CBF after encoding is $h(x, t) = \widetilde{\text{min}}(h_1(x, t), \hdots, h_p(x, t))$ where each $h_i$ corresponds to a temporal formula $\phi_i = \mathcal{T}_{[t_i, t_{i + 1}]}\varphi$ or $\phi_i = \varphi_1\mathcal{T}_{[t_i, t_{i + 1}]}\varphi_2$ where $\mathcal{T}$ is a temporal operator. Note that after $t_{i + 1}$, $\phi_i$ no longer needs to be enforced and thus $h_i$ can be deleted from the set of arguments to $\widetilde{\text{min}}$ after $t_{i + 1}$. We consider the following deletion strategy: Let $\mathcal{A}_{t_0}(x, t) \coloneq \{h_1(x, t), \hdots, h_p(x, t)\}$ and $\mathcal{A}_{t_{i+1}+1}(x,t) := \mathcal{A}_{t_i} \setminus \{h_i(x,t)\}$ and we set $h(x, t) \coloneq \widetilde{\text{min}}(\mathcal{A}_{t_i}(x, t))$ when $t_i \le t \le t_{i + 1}$. When several functions $h_j$ corresponds to the same switching time, the deletion strategy can remove all such terms simultaneously. We present Theorem \ref{thm:safety_cbf}\appendixswitch{, whose proof is in Appendix \ref{proof:safety_cbf}}.

\begin{theorem}
\label{thm:safety_cbf}
    \textbf{Safety from STL CBF.} Given a formula $\phi$ following \eqref{eq:syntax}. Let $x_0 \in \mathcal{C}_0$. Suppose there exists a valid CBF in the sense of \eqref{eq:cbf_cons} constructed satisfying the conditions in Table \ref{tab:stl_encodings}. Then, enforcing \eqref{eq:cbf_cons} in Theorem \ref{thm:forward} implies $(x, 0) \models \phi$.
\end{theorem}

Theorem \ref{thm:safety_cbf} applies to a controller that enforces the exact constraint \eqref{eq:cbf_cons}. It does not directly establish STL satisfaction for Algorithm \ref{alg:algorithm}, which replaces the constraint with a first order approximation, which we discuss later in \eqref{eq:cons_replaced}. Consider an example with $x\in \mathbb{R}$, $x_0 = -2$, and $\phi \coloneq G_{[0, 10]}x \ge -4 \wedge F_{[0, 5]}x \ge 0$. We construct a candidate CBF following the rules of Table \ref{tab:stl_encodings} where $h^\phi(x,t) \coloneq \widetilde{\text{min}}(x + 4, \kappa(x, t))$ with $\kappa(x, t) \coloneq x + 5 - t$ removed from the arguments of $\widetilde{\text{min}}$ at $t = 6$. Note that our construction has no guarantee that $h^\phi$ is valid. We make Assumption \ref{assumption:validity}, standard for CBF in the MPPI literature \cite{parwana2025br, yin2023shield}.

\begin{assumption}
\label{assumption:validity}
    $S(x_t, t) \neq \emptyset, \forall t \in \{0, \hdots,  H- 1\}$.
\end{assumption}
It is easy to show that $S(x_t, t)$ is nonempty if $\mathcal{U} = \mathbb{R}^m$ and $f$ is feedback-equivalent to a single integrator $x_{t + 1} = x_t + u_t$ when $C_t$ is nonempty. Our algorithm requires $\frac{\partial h^\phi}{\partial x}g$ to be not identically zero (i.e., the CBF is of $1$st order). For formulas that result in high-order CBFs, we reduce the order by substituting the forward dynamics following \cite{parwana2025br}. See example in Case D of Section \ref{subsec:artificial}.
\subsection{Model Predictive Path Integral}
MPPI is a receding-horizon sampling-based strategy, which we use to solve for $v^*$ in Problem \ref{prob:formulation}. Given a fixed horizon $\mathcal{H}$, at each time $t$, we sample $K$ control input perturbations $\omega^{k, (t)} = (\omega^{k, (t)}_t, \hdots, \omega^{k, (t)}_{t+ \mathcal{H} - 1})$ where $\omega^{k, (t)} \in \mathbb{R}^m \times \mathcal{H}$ for $k \in \{1, \hdots, K\}$, where $K$ is a sampling size. Further, we require $\omega_\tau^{k, (t)} \sim \mathcal{N}(0, \Sigma_\omega)$ to be sampled from a zero-centered Gaussian distribution with a noise covariance of $\Sigma_\omega$. Based on each sample, we can obtain a sample rollout $x^k$ up to time $t + \mathcal{H}$ using a sample control sequence $u^{k, (t)} = v^{(t)} + \omega^{k, (t)} \in \mathbb{R}^{m \times \mathcal{H}}$ where $v^{(t)}$ is a nominal control input. We can then compute a sample cost
\begin{align}
    \label{eq:sample_cost}
    J^{k, (t)} \coloneq& E(x^{k, (t)}_{t + \mathcal{H}}) + \sum_{\tau = t}^{t + \mathcal{H} - 1}[q(x^{k, (t)}_\tau, u^{k, (t)}_\tau) \\&+\gamma (v^{(t)}_\tau)^T\Sigma_\omega^{-1}\omega_\tau^{k, (t)}], \nonumber
\end{align}
where\footnote{Note that some work \cite{parwana2025br, yin2023shield} replace $\gamma (v^{(t)}_\tau)^T\Sigma_\omega^{-1}\omega_\tau^{k, (t)}$ with other action perturbation cost functions in \eqref{eq:sample_cost}. We use \eqref{eq:sample_cost} for its consistency with the original formulation in MPPI \cite{williams2018information} and use \eqref{eq:sample_cost} consistently across baselines in Section \ref{sec:case_study} for fair comparison.} $\gamma \ge 0$ is a regularization factor. Note that inherently $J^{k, (t)}$ does not contain any safety information pertinent to $\phi$. As inspired by \cite{yin2023shield}, we incorporate the CBF information in the cost function by considering the safety augmented cost, 
\begin{align*}
    J^{k, (t)}_\phi \coloneq J^{k, (t)} + \eta \sum_{\tau = t}^{t + \mathcal{H} - 1} C_{\textup{CBF}}(x^{k, (t)}, \tau), 
\end{align*} 
where $C_{\textup{CBF}}(x^{k, (t)}, \tau) \coloneq \max\{0, (1 - \alpha)h^\phi(x^{k, (t)}_{\tau}, \tau)-h^\phi(x^{k, (t)}_{\tau + 1}, {\tau + 1})\}$ denotes the CBF constraint violation. We remark here $\eta \ge 0$ is a tunable hyperparameter. We also remark that one can consider the robust semantics $\rho^\phi$ as a part of the cost (as considered in \cite{baldini2024don}), but for general STL this requires a shrinking-horizon rollout where $\mathcal{H} \coloneq H - t$, which has a high computational cost. Standard in literature, the next step is to assign a weight on each of the sample cost based on the performance using $W^{k, (t)} \coloneq \exp[-\frac{1}{\lambda}(J^{k, (t)}_\phi - \min_{i \in \{1, \hdots, K\}}J^{i, (t)}_\phi)]$, with $\lambda\ge \gamma$, and the optimal control sequence is expressed as a weighted rollout $v^{+, (t)} = \sum_{k = 1}^KW^{k, (t)}u^{k, (t)}/\sum_{k = 1}^KW^{k, (t)}$. Despite $C_{\textup{CBF}}$, $v^{+, (t)}$ is not guaranteed to satisfy the CBF constraint. Standard procedure \cite{yin2023shield} considers a safety filter as adapted below for our setting of time-varying CBF,
\begin{align}
\label{eq:cbf_opt}
    &{v^s_\tau}^{*, (t)} \coloneq \arg\min_{v^s_\tau \in \mathcal{U}} \|v^{+, (t)}_\tau - v^s_\tau\|\\
    \text{s.t. }& h^\phi(f(x^{(t)}_\tau, v^s_\tau), \tau + 1) - h^\phi(x^{(t)}_\tau, \tau) \ge -\alpha h^\phi(x^{(t)}_\tau, \tau), \nonumber
\end{align}
where $\tau \in \{t, \hdots, t + N\}$ and $0 \le N \le \mathcal{H} - 1$. Let ${v^s_\tau}^{*, (t)} = v_\tau^{+, (t)}$ for $\tau  > t + N$. One can forward ${v^s_t}^{*, (t)}$ to the dynamics and plan in receding-horizon with MPPI. One can use ${v_{t + 1}^s}^{*, (t)}, \hdots, {v_{t + \mathcal{H} - 1}^s}^{*, (t)}$ as a warm-start for time $t + 1$.  Note that \eqref{eq:cbf_opt} is generally nonconvex in discrete-time. The standard procedure proposed in \cite{yin2023shield} to approach \eqref{eq:cbf_opt} is to solve the following through a first-order optimizer ${v_{t:t + N}^s}^{*, (t)} \approx \arg\max \sum_{\tau = t}^{t + N} \min\{0, h^\phi(x^{(t)}_{\tau + 1}, \tau + 1) + (\alpha - 1)h^\phi(x^{(t)}_\tau, \tau)\}$
which guarantees safety when the cost is zero and when $\mathcal{U} = \mathbb{R}^m$. Usually this is not run until convergence in real-time. Moreover, the solution is not guaranteed optimal to \eqref{eq:cbf_opt} even after convergence. Alternatively, a closed-form solution is proposed in \cite{parwana2025br} for time-invariant CBF with $\mathcal{U} = \mathbb{R}^m$ using first-order Taylor expansion. Motivated by \cite{parwana2025br}, we propose an approximated quadratic program solved in closed-form when $\mathcal{U} = \mathbb{R}^m$.
\begin{algorithm}
	\caption{Safety-Aware-STL-MPPI.} \label{alg:algorithm} 
	\begin{algorithmic}[1]
        \State We are given $\phi$, $\mathcal{H}$, $f$, $E$, and $q$. 
        \State We have constructed $h^\phi$ following Section \ref{subsec:stl_cbf} before the closed-loop planning.
        \State Suppose the current time is $t$. We are given hyperparameters $\Sigma_\omega, K, \eta, \gamma, \lambda, N$.
        \For{$k \in \{1, \hdots, K\}$ \text{ in parallel }}
        \State Sample $\omega^{k, (t)} \gets(\omega^{k, (t)}_t, \hdots, \omega^{k, (t)}_{t + \mathcal{H} - 1})$,
        \State where $\omega^{k, (t)}_\tau \sim \mathcal{N}(0, \Sigma_\omega)$.
        \State $u^{k, (t)} = v^{(t)} + \omega^{k, (t)}$; $J^{k, (t)}_\phi \gets 0$.
        \For{$\tau \in \{t, \hdots, t + \mathcal{H} - 1\}$}
        \State $x_{\tau + 1}^{k, (t)} \gets f(x_\tau^{k, (t)}, u^{k, (t)}_\tau)$.
         \State $\begin{aligned}
            J^{k,(t)}_\phi \gets &J^{k,(t)}_\phi +
             q(x_\tau^{k,(t)},\, u_\tau^{k,(t)}) \\&
            + \gamma\,(v^{(t)}_\tau)^T \Sigma_\omega^{-1}\omega_\tau^{k,(t)} \\&+ \eta\,C_{\text{CBF}}(x^{k,(t)},\,\tau).
\end{aligned}$
        \EndFor
        \State $J^{k, (t)}_\phi \gets J^{k, (t)}_\phi + E(x_{t + \mathcal{H}}^{k, (t)})$.
        \EndFor
        \State $W^{k, (t)} \coloneq \exp[-\frac{1}{\lambda}(J^{k, (t)}_\phi - \min_{i \in \{1, \hdots, K\}}J^{i, (t)}_\phi)]$
        \State $v^{+, (t)} = \sum_{k = 1}^KW^{k, (t)}u^{k, (t)}/\sum_{k = 1}^KW^{k, (t)}$.
        \State $\{\hat{v}^{s, (t)}_t,\hdots, \hat{v}^{s, (t)}_{t + N}\} \gets \eqref{eq:qp} \text{ or } \eqref{eq:qp_closed}$.
        \State Warm start with the new weighted rollout.
        \State Execute $\hat{v}^{s, (t)}_t$.
	\end{algorithmic} 
\end{algorithm}
\subsection{Solving the CBF Constraint with Approximation}
\label{subsec:approx}
\begin{figure*}[t]
  \centering
  \begin{subfigure}{0.245\textwidth}
    \centering
    \includegraphics[width=\linewidth]{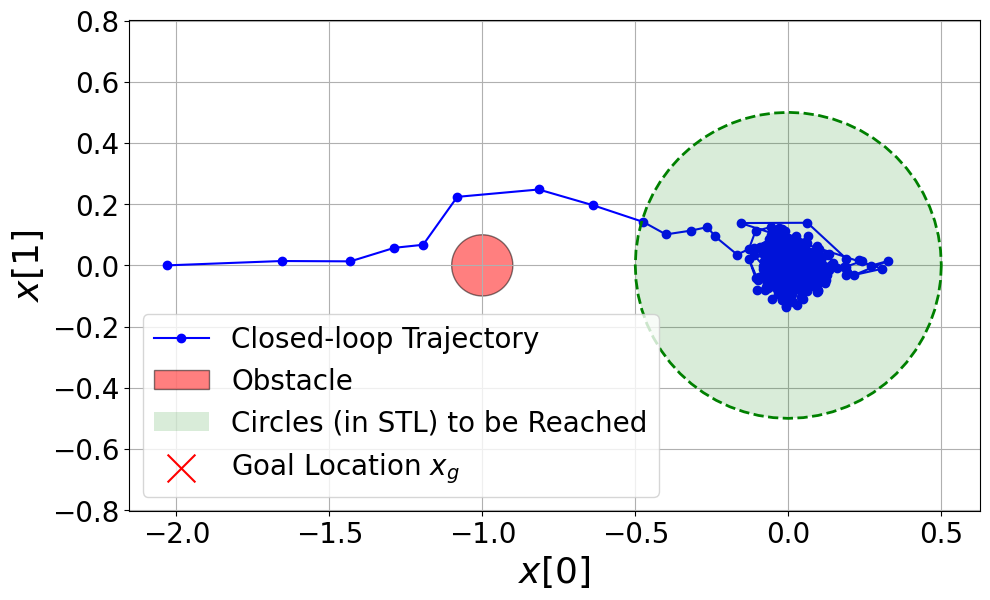}
    \caption{Case A}
    \label{fig:caseA}
  \end{subfigure}
  \begin{subfigure}{0.245\textwidth}
    \centering
    \includegraphics[width=\linewidth]{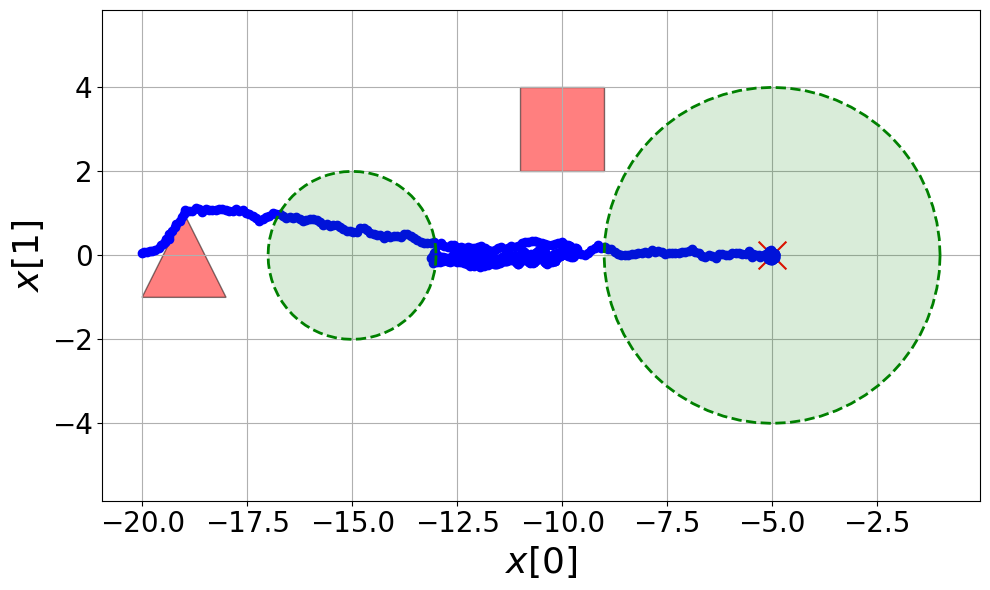}
    \caption{Case B}
    \label{fig:caseB}
  \end{subfigure}
  \begin{subfigure}{0.245\textwidth}
    \centering
    \includegraphics[width=\linewidth]{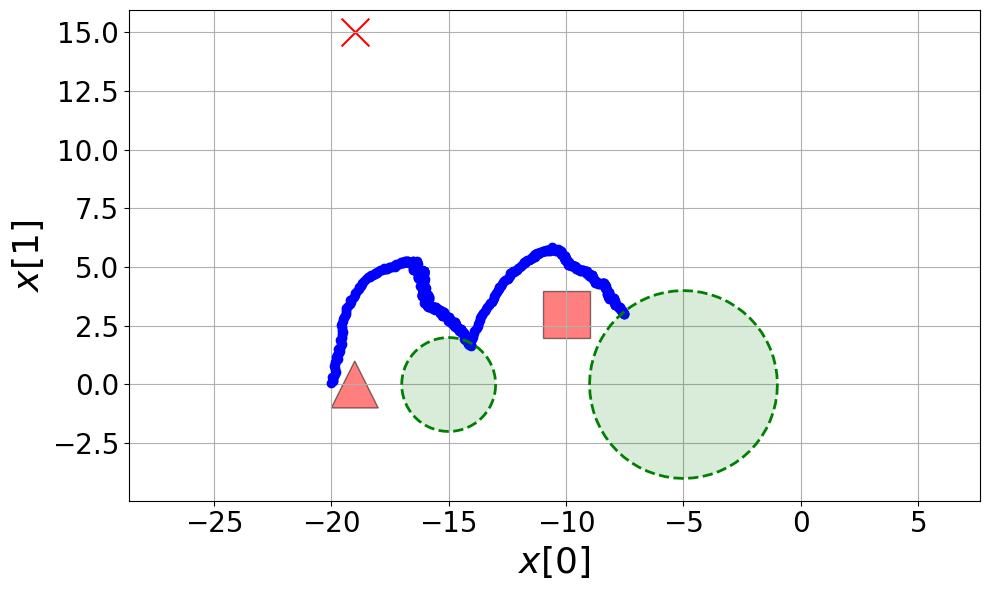}
    \caption{Case C}
    \label{fig:caseC}
  \end{subfigure}
  \begin{subfigure}{0.245\textwidth}
    \centering
    \includegraphics[width=\linewidth]{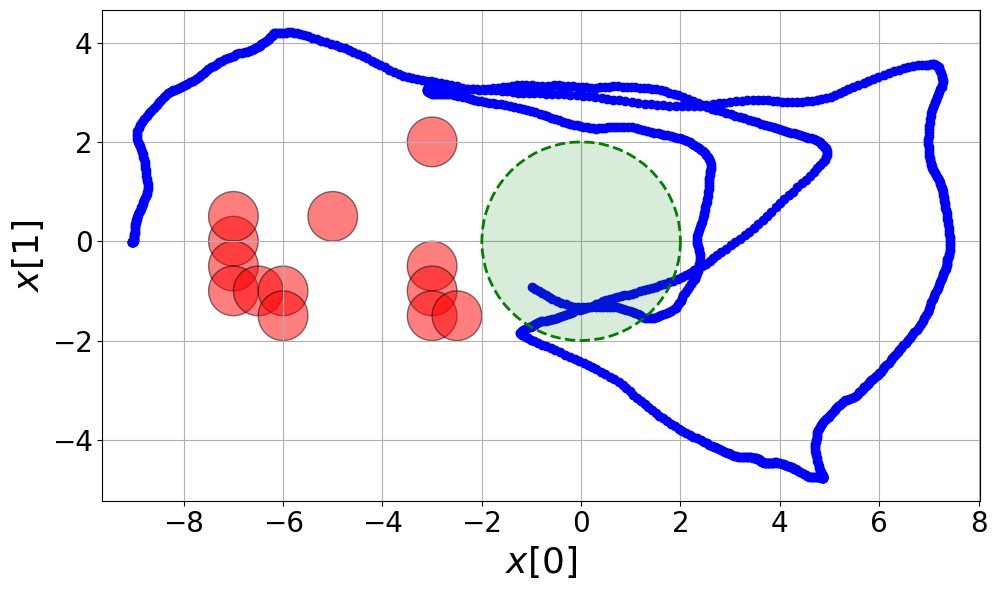}
    \caption{Case D}
    \label{fig:caseD}
  \end{subfigure}
  \caption{Example Trajectories Planned with safety-aware-stl-mppi (Limited Actuation).}
  \label{fig:succ_demo_cases}
  \vspace{-10pt}
\end{figure*}

Since $h^\phi(x, t)$ is differentiable in $x$, we have by first-order Taylor expansion around $x_
\tau$ that $h^\phi(f(x^{(t)}_\tau, v_\tau^s), \tau + 1) \approx h^\phi(x^{(t)}_\tau, \tau + 1) + \left.\frac{\partial h^\phi}{\partial x}^T\right|_{(x^{(t)}_\tau, \tau + 1)}[f(x_\tau^{(t)}, v^s_\tau) - x^{(t)}_\tau]$. Suppose $A(x_\tau^{(t)}, \tau) \coloneq \left.\frac{\partial h^\phi}{\partial x}\right|_{(x_\tau^{(t)}, \tau + 1)} \in \mathbb{R}^n$. Given the input-affine dynamics $f$, we can approximate the constraint in \eqref{eq:cbf_opt} with
\begin{align}
    \label{eq:cons_replaced}
    &h^\phi(x_\tau^{(t)}, \tau + 1) + A(x_\tau^{(t)}, \tau)^T[(f_0(x_\tau^{(t)}) - x_\tau^{(t)}) \\&+ g(x_\tau^{(t)})v^s_\tau]  \ge (1 -\alpha) h^\phi(x_\tau^{(t)}, \tau), \nonumber
\end{align}
which is affine in $v_\tau^s$. We can therefore attain 
\begin{align}
    \label{eq:replaced}
    \hat{v}^{s, (t)}_\tau\coloneq \arg\min_{v^s_\tau \in \mathcal{U}} \{\|v^{+, (t)}_\tau - v^s_\tau\| \text{ subject to } \eqref{eq:cons_replaced}\},
\end{align}
where ${v^s_\tau}^{*, (t)} \approx \hat{v}^{s, (t)}_\tau$. Consider all active actuation limits $e_i^Tv_\tau^s \le u_{\max, i}$ and $-e_j^Tv_\tau^s \le -u_{\min, j}, \forall (i, j) \in \mathcal{A}$ where $\mathcal{A} \subseteq \{1, \hdots, m\}^2$. Let $\Pi \coloneq (a_{\textup{cbf}}^T(x_\tau^{(t)}, \tau), e_i^T, \hdots,-e_j^T, \hdots)^T$ and $b \coloneq (b_{\textup{cbf}}(x_\tau^{(t)}, \tau), u_{\max, i}, \hdots, -u_{\min, i}, \hdots)$ where 
\begin{align*}
    &a_{\textup{cbf}}^T(x_\tau^{(t)}, \tau) \coloneq -A(x_\tau^{(t)}, \tau)^Tg(x_\tau^{(t)});\\&
    b_{\textup{cbf}}(x_\tau^{(t)}, \tau) \coloneq A(x_\tau^{(t)}, \tau)^T(f_0(x_\tau^{(t)}) - x_\tau^{(t)}) +\\&h^\phi(x_\tau^{(t)}, \tau + 1) + (\alpha - 1) h^\phi(x_\tau^{(t)}, \tau).
\end{align*} 
Then, the constraints in \eqref{eq:cons_replaced} is equivalent to $\Pi v_\tau^s \le b$. We can therefore write $\eqref{eq:replaced}$ in standard Quadratic Program (QP)
\begin{align}
\label{eq:qp}
    \hat{v}^{s, (t)}_\tau =& \arg\min\{\frac{1}{2}{(v_\tau^s - v_\tau^{+, (t)})}^TP(v_\tau^s - v_\tau^{+, (t)}) \\&\mid \Pi v_\tau^s \le b\}, \nonumber
\end{align}
where $P \coloneq I_m$ is the $m$ dimensional identity matrix. The QP in \eqref{eq:qp} can be solved efficiently at runtime and admits a closed form solution when $\mathcal{U} = \mathbb{R}^m$, in which case we have
\begin{align}
\label{eq:qp_closed}
    \hat{v}^{s, (t)}_\tau =& v_\tau^{+, (t)} - a_{\textup{cbf}}(x_\tau^{(t)}, \tau)[\max\{0, a_{\textup{cbf}}^T(x_\tau^{(t)}, \tau)v_\tau^+ \\&- b_{\textup{cbf}}(x_\tau^{(t)}, \tau)\}/{\|a_{\textup{cbf}}(x_\tau^{(t)}, \tau)\|_2^2}].\nonumber
\end{align}

The derivation from \eqref{eq:qp} to \eqref{eq:qp_closed} when under unlimited actuation is standard, but we present it for completeness in Appendix \ref{appendix:qp}. We summarize the safety-aware-stl-mppi in Algorithm \ref{alg:algorithm}, which we call in closed-loop. Naturally, along the closed-loop trajectory, we assume $\hat{S}(x_\tau, \tau) \neq \emptyset$ where $\hat{S}(x, \tau) \coloneq \{v \in U \mid a_{\text{cbf}}(x, \tau)^Tv \le b_{\text{cbf}}(x, \tau)\}$ and that $a_{\text{cbf}}(x_\tau^{(t)}, \tau) \neq 0$ at every state where \eqref{eq:qp_closed} is applied. Due to the approximation error, perfect safety is not guaranteed theoretically.\appendixswitch{ We discuss how one can address this by optional robustification in Theorem \ref{thm:tighten_local} in the Appendix.}{ We discuss how one can address this by robustification in the full version of this paper \cite{}.}
\begin{table*}[t]
\centering

\begin{subtable}{\textwidth}
\centering
\caption{Limited actuation}
\label{tab:limited_act}
\begin{tabular}{l|cc|cc|cc|cc}
\toprule
 & \multicolumn{2}{c|}{Case A}
 & \multicolumn{2}{c|}{Case B}
 & \multicolumn{2}{c|}{Case C}
 & \multicolumn{2}{c}{Case D} \\
\cmidrule(lr){2-9}
Method
& SR $\uparrow$ & Timing $\downarrow$
& SR $\uparrow$ & Timing $\downarrow$
& SR $\uparrow$ & Timing $\downarrow$
& SR $\uparrow$ & Timing $\downarrow$ \\
\midrule
\rowcolor{blue!7}
Safety-Aware-STL-MPPI (Ours)   & 100\% & 5.26 & 100\% & 8.80 & 100\% & 8.84 & 100\% &17.93 \\
Vanilla MPPI           & 100\% & 1.35 & 100\% & 3.32 & 0\% & 3.33 & 0\% & 5.97 \\
Reach-avoid MPPI       & 100\% & 5.92 & 100\% &7.31  & 0\% &  7.32 & 25\%  & 17.05 \\
Robustness-based MPPI  &  100\% & 190.35 & 95\% & 90.74 & 90\% & 90.84 & 55\% & 1060.55 \\
Penalty-based MPPI     & 100\% & 191.16 & 80\% & 91.21 & 85\% &  91.34 & 60\% &  1057.51 \\
\bottomrule
\end{tabular}
\end{subtable}

\begin{subtable}{\textwidth}
\centering
\caption{Unlimited actuation}
\label{tab:unlimited_act}
\begin{tabular}{l|cc|cc|cc|cc}
\toprule
 & \multicolumn{2}{c|}{Case A}
 & \multicolumn{2}{c|}{Case B}
 & \multicolumn{2}{c|}{Case C}
 & \multicolumn{2}{c}{Case D} \\
\cmidrule(lr){2-9}
Method
& SR $\uparrow$ & Timing $\downarrow$
& SR $\uparrow$ & Timing $\downarrow$
& SR $\uparrow$ & Timing $\downarrow$
& SR $\uparrow$ & Timing $\downarrow$ \\
\midrule
\rowcolor{blue!7}
Safety-Aware-STL-MPPI (Ours)   & 100\% & 3.80 & 100\% & 7.40 & 100\% & 7.37 & 100\% &16.39 \\
Vanilla MPPI           & 100\% & 1.55 & 100\% & 3.31 & 0\% & 3.31 & 0\% & 5.90 \\
Reach-avoid MPPI       & 100\% & 3.47 & 100\% &5.19  & 0\% &  5.21 & 25\%  & 14.55 \\
Robustness-based MPPI  &  100\% & 202.87 & 100\% & 90.32 & 100\% & 90.31 & 55\% & 1057.29 \\
Penalty-based MPPI     & 95\% & 203.45 & 100\% & 90.65 & 95\% &  90.59 & 60\% &  1057.07 \\
\bottomrule
\end{tabular}
\end{subtable}
\caption{Section \ref{subsec:artificial} Evaluation Results.}
\label{tab:mars_rov}
\vspace{-15pt}
\end{table*}
\section{Evaluation}
\label{sec:case_study}
We evaluate safety-aware-stl-mppi via four artificial Mars rover path planning case studies and a photorealistic simulation with a quadcopter. In the rover case studies, we compare the algorithm with four state-of-the-art baseline models (detailed in Section \ref{subsec:artificial}) and discuss our advantage in safety and low computation cost. We then simulate a drone inspection task in Nvidia Isaac Lab \cite{makoviychuk2021isaac, mittal2023orbit} with safety-aware-stl-mppi. All experiments are run on a Dell G16 7630 Model with a 13th Gen Intel(R) Core(TM) i9 processor.

\subsection{Mars Rover Case Studies}
\label{subsec:artificial}
We consider artificial case studies where we simulate safety-critical missions of a Mars rover. In each mission, the rover is subject to a safety constraint defined in STL (e.g., to return to home base before the battery dies out and to avoid hazardous regions at all time). We consider four case studies with different dynamics, goals, and safety specifications, which we evaluate with and without actuation limit. In each of cases A, B, and C, we consider a goal location to inspect, incentivized by the terminal cost $E(x^{k, (t)}_{t+ \mathcal{H}}) \coloneq \|x^{k, (t)}_{t+ \mathcal{H}} - x_g\|_2$ where $x_g \in \mathbb{R}^2$ denotes the goal location. In case D, we simply set $E(x_{t+\mathcal{H}}^{k, (t)}) \coloneq 0$. Let $q(x^{(t)}_\tau, u_\tau) \coloneq 0$ for all cases. The purpose of the hypothetical goal locations is to show that safety-aware-stl-mppi is robust to adversarial (and non-adversarial) cost objectives. We expect safety-aware-stl-mppi to respect the safety constraint even if the goal is not achievable in due time. All unit of time are seconds, which we highlight with the $s$ postfix in the metric interval of the STL specifications. Since the experiment runs in real-time, we set a sampling time of $\Delta t \coloneq 0.01s$ between the two time index $t$ and $t + 1$. In all cases, the running costs are $0$. For validation, we draw initial positions randomly as specified below and show examples of successful trajectories from safety-aware-stl-mppi for all cases (with actuation limits) in Figure \ref{fig:succ_demo_cases}\footnote{We emphasize that in Figure \ref{fig:caseB}, we plot exact traingle and squares instead of smooth approximations.}. 

\begin{itemize}
    \item Case A: Consider a nonlinear dynamics $x_{t+1} = \sin(x_t) + u_t$ where $x_t, u_t \in \mathbb{R}^2$ and $x_0 \sim (U(-3, -2), 0)^T$ where $U(a, b)$ denotes the uniform distribution over $[a, b)$ and $x_g \coloneq(0, 0)$. We are given the safety specification $\phi_A \coloneq G_{[0s, 10s]}\|x-(-1,0)^T\|_2 \ge 0.1 \wedge
        F_{[0s, 10s]}0.5 \ge \|x - (0, 0)^T\|_2,$ specifying the robot to at all time avoid a circle of radius $0.1$ centered at $(-1, 0)^T$ and reach a circle of radius $0.5$ centered at the origin between $0$ and $10$ seconds. The actuation limit is $(-5, -5)^T \preccurlyeq u_t \preccurlyeq (5, 5)^T$.
    \item Case B: Consider a single integrator dynamics $x_{t + 1} = x_t + u_t$ where $x_t, u_t \in \mathbb{R}^2$, $x_0 \sim (-20, U(0, 2))^T$ and a goal location $x_g \coloneq (-5, 0)^T$. We are given the following safety specification, $\phi_B \coloneq G_{[0s, 10s]}T(x) \ge 0 \wedge G_{[0s, 10s]}S(x) \ge 0\wedge F_{[0s, 7s]}4 - [(x[1]^2+(x[0] + 15)^2] \ge 0\wedge F_{[0s, 10s]}16 - [x[1]^2 + (x[0] + 5)^2] \ge 0,$
    where $T(x) = \widetilde{\text{max}}(x[1] + 2x[0] + 37, -1-x[1], -2x[0]+ x[1]-39)$, $S(x) \coloneq \widetilde{\text{max}}(-11-x[0], x[0] + 9, x[1] - 4, 2 - x[1])$, and $\widetilde{\text{max}}(a_1, a_2,\hdots) \coloneq \frac{1}{\Omega} \log (\exp(\Omega a_1) + \exp(\Omega a_2) + \hdots)$ with $\Omega > 0$ denotes smooth maximum. We set $\Omega\coloneq 1000$. The formula specifies that the robot should always avoid a triangle with vertices $(-20, -1)^T, (-19, 1)^T, (-18, -1)^T$ and a square with vertices $(-11, 2)^T, (-9, 2)^T, (-11, 4)^T, (-9, 4)^T$. We require the robot to reach a circle of radius $2$ centered at $(-15, 0)^T$ in time interval $[0s, 7s]$ and a circle of radius $4$ centered at $(-5, 0)^T$ in time interval $[0s, 10s]$. The actuation limit is $(-0.08, -0.08)^T \preccurlyeq u_t \preccurlyeq (0.08, 0.08)^T$.
    \item Case C: Case C is identical to Case B but with the goal location $x_g \coloneq (-19, 15)^T$.
    \item Case D: We consider a double integrator dynamics \(x_{t+1} =  
    \begin{bmatrix} p_{t+1} \\ v_{t+1} \end{bmatrix}
    =\begin{bmatrix} I_2 & I_2 \\ 0 & I_2 \end{bmatrix}
    \begin{bmatrix} p_t \\ v_t \end{bmatrix}
    +
    \begin{bmatrix} 0 \\ I_2 \end{bmatrix} u_t,
    \) where $I_2$ is a $2-$dimensional identity matrix. We consider $x_0 \sim (U(-9.5, -9), 0, 0, 0)^T$ and $13$ circular obstacles, shown in Figure \ref{fig:caseD}. Let the centers of each obstacle $i$ be $c^i$, and we consider the safety specification $\phi_D \coloneq (\wedge_{i = 1}^{13}G_{[0s, 10s]}\|p - c^i\| \ge 0.5) \wedge F_{[9s, 10s]} 2 - \|p - (0, 0)^T\| \ge 0$, which means that the rover needs to avoid all obstacles at all time and reach a circle of radius $2$ centered at the origin in time interval $[9s, 10s]$. The actuation limit is $(-20, -20)^T \preccurlyeq u_t \preccurlyeq (20, 20)^T$.
\end{itemize}
We construct candidate CBFs:
\begin{align*}
    &h^{\phi_A}(x, t) \coloneq \widetilde{\text{min}}((x_t[0] + 1)^2 + x_t[1]^2 - 0.1^2,\\ &-\frac{7}{8}t\Delta t + 9-(x_t[0]^2 + x_t[1]^2)),\\
    &h^{\phi_B}(x, t) \coloneq \widetilde{\text{min}}(\widetilde{\text{min}}(T(x_t), S(x_t)), \widetilde{\text{min}}(-\frac{32}{7}
t\Delta t+36-\\&x_t[1]^2-(x[0]+15)^2,-24t\Delta t+256-(x_t[0]+5)^2-x_t[1]^2)),\\
    &h^{\phi_D}(x, t) \coloneq \widetilde{\text{min}}(\widetilde{\text{min}}_{i = 1}^{12}((x_{t+1}[0] - c^i[0])^2 \\&+ (x_{t+1}[1] - c^i[1])^2-0.5^2, (x_{t+1}[0] - c^{i + 1}[0])^2 + (x_{t+1}[1]\\& - c^{i + 1}[1])^2-0.5^2),100 - \frac{96t\Delta t}{10 - \Delta t} - (x_{t+1}[0]^2 + x_{t+1}[1]^2)),
\end{align*}
where $\widetilde{\text{min}}_{i = 1}^{12}$ means that $\widetilde{\text{min}}$ is applied $12$ times. We switch $h^{\phi_B}$ to $\widetilde{\text{min}}(\widetilde{\text{min}}(T(x_t), S(x_t)), -24t\Delta t+256-(x_t[0]+5)^2-x_t[1]^2)$ at $t\Delta t =7$\footnote{The reported results instantiate the switch at $t\Delta t = 7$ rather than $7 + \Delta t$; this affects only the final discretization step of the $F_{[0s,7s]}$ component. Reported robustness values are computed by an independent monitor under the exact closed-interval semantics and are therefore unaffected.}. Note that $h^{\phi_C} \coloneq h^{\phi_B}$. Case C is more difficult than case B in that $x_g$ may not be achievable given the specification, and the robot needs to obey the safety constraint even if it may not reach close to the goal location.
\begin{figure*}[t]
    \centering
  \begin{subfigure}[t]{0.257\textwidth}
    \centering
    \includegraphics[max width=\textwidth,keepaspectratio]{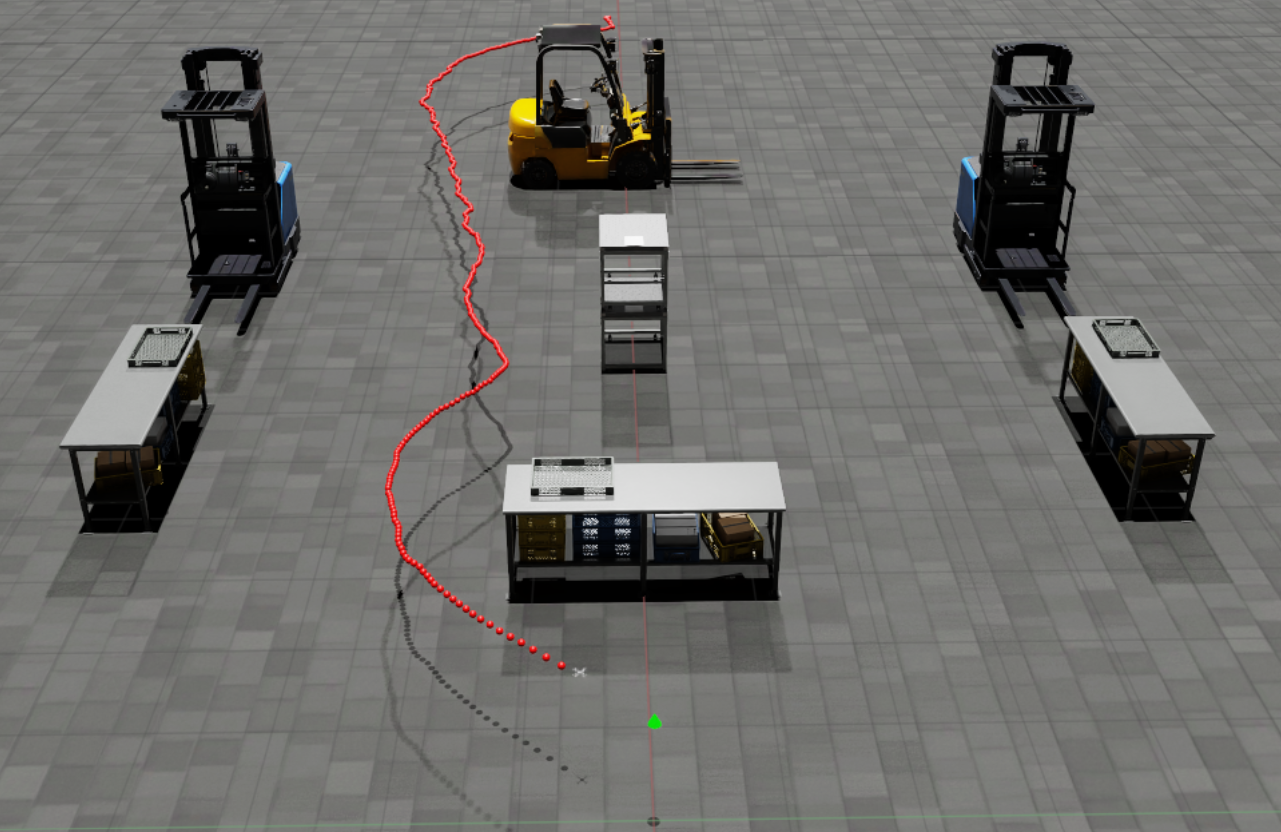}
    \caption{Example Planned Trajectory}
    \label{fig:isaac_planned}
  \end{subfigure}\hspace{0.01\textwidth}
  \begin{subfigure}[t]{0.323\textwidth}
    \centering
    \includegraphics[max width=\textwidth,keepaspectratio]{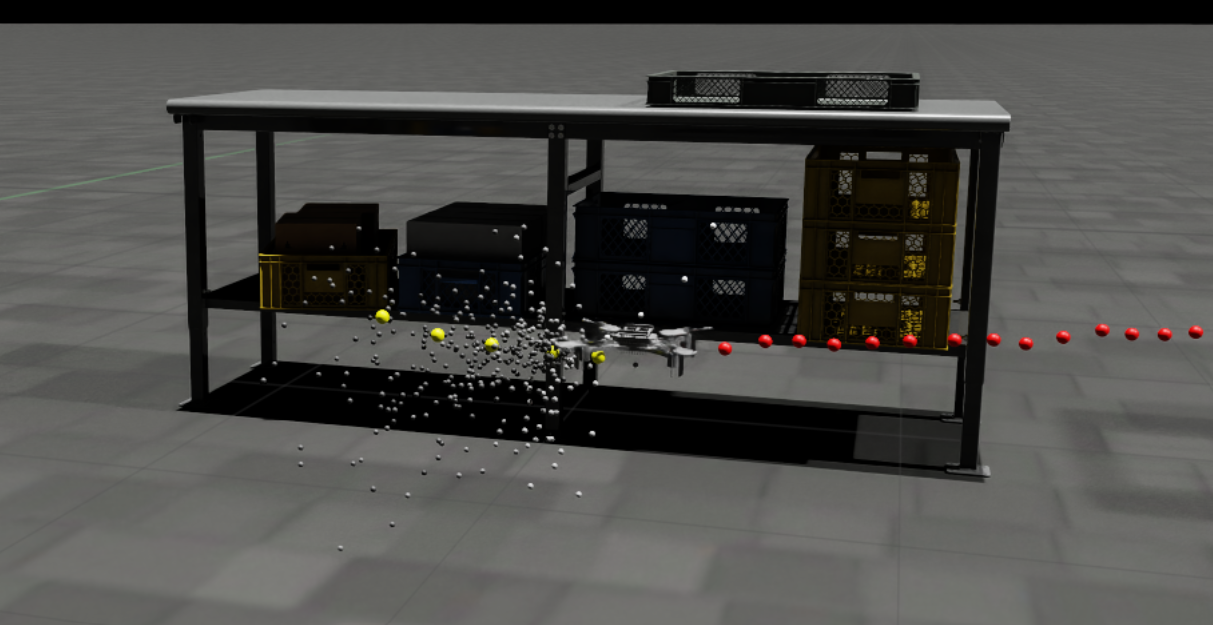}
    \caption{Example Planned Rollouts}
    \label{fig:isaac_rollout}
  \end{subfigure}
  \caption{Quadcopter Simulation Example}
  \label{fig:isaac}
  \vspace{-15pt}
\end{figure*}

\textbf{Baseline Models.} We compare our results with four other baseline models, all adapted to PyTorch \cite{paszke2019pytorch} for consistent comparison. In \textbf{vanilla MPPI} \cite{williams2016aggressive}, we consider constant penalties respectively for each violation of obstacle avoidance and for when the robot does not satisfy the reach requirement at time $\tau$, where $\tau \in \{t_1, \hdots, t_2\}$ with $[t_1, t_2]$ denoting the STL metric interval for each time operator in sampling index\footnote{In Case D, obstacle-avoidance predicates for Vanilla MPPI are evaluated on a one-step-ahead kinematic surrogate, consistent with the CBF treatment for higher-order dynamics.}. In \textbf{reach-avoid MPPI} \cite{parwana2024model}, we consider a cost function $c_r(x^k_\tau) \coloneq k_r (\|p_t- p_r\|^2 - r_r^2)$ for each center of the circles to be reached $p_r$ with a sufficient radius $r_r$ and a cost function $c_o(x^k_\tau)\coloneq k_o/(\max(\|p_t - p_o\|, \epsilon_o)$ for each obstacle center $p_o$ in the STL specifications. Note that $p_r$ is not to be confused with $x_g$. A composite cost function $Q(x_{t:t + \mathcal{H}}) \coloneq \max(Q_{r_1}, Q_{r_2}, \hdots, Q_{o_1}, Q_{o_2})$, where $\{r_1, \hdots\}$ and $\{o_1, \hdots\}$ denote the centers of circles to be reached and of obstacles in the STL, is considered with $Q_o \coloneq \max\{c_o(x_{t_1}), \hdots, c_o(x_{t_2})\}$ and $Q_r \coloneq \min\{c_r(x_{t_1}), \hdots, c_r(x_{t_2})\}$, where $[t_1, t_2]$ are metric intervals for the time operators\footnote{For numerical robustness in discrete-time, we adopt a one-sample temporal buffer when $t_1 > 0$, evaluating $Q_r$ over $[t_1 - \Delta t, t_2]$.}. We add the composite cost function to the rollout costs. We use $\nabla_xh(f(x_t, u_t))(f(x_t, u_t)- x_t) \ge  - \alpha h(x_t)$ as a time-invariant CBF filter. In \textbf{penalty-based MPPI} (i.e. MPPI-STL) \cite{baldini2024don}, a penalty function, $I(x^k)$, for the violation of the safety specification based on STL robustness is augmented to the rollout costs. Before calling MPPI, the paper relies on a hybrid $A^*$ guiding prior factored into the cost function. We disable this pre-planning step to compare the core MPPI algorithms under identical priors. The paper did not specify the exact penalty, so we choose $I(x^k) \coloneq -\rho^\phi(x^k) 1(x^k \not\models \phi)$, where $1$ is an indicator function and $\rho^\phi$ is the robust semantics over $\phi$. An open loop path integral optimization method, which we here denote as \textbf{robustness-based MPPI}, is proposed in \cite{halder2025trajectory} with $I(x^k)\coloneq -\rho^\phi(x^k)$. The paper also proposes a hyperparameter tuning procedure, which we have to disregard for a real-time closed-loop implementation. Note that both penalty-based and robustness-based MPPIs require the computation of STL robust semantics, which must be executed in shrinking-horizon (i.e., $\mathcal{H} \coloneq H - t + 1$ adapts online). The augmented costs for each method are rescaled with weights\footnote{For all methods that do not use a CBF projection, as standard practice \cite{um-arm-lab_pytorch_mppi_2024}, we impose an action filter in the rollouts to bound the actions within actuation limit during the rollouts.}. We use the same regularization and weighting schemes across baselines. We specifically emphasize that none of the above methods are designed for providing safety guarantee with STL under arbitrary cost functions, and to the best our knowledge we are the first work to address such problem with MPPI.

\textbf{Results.} For each case, we evaluate the experiment via three metrics. \textbf{Success Rate (SR)} measures the percentage of closed-loop planned trajectories over the $20$ trials with nonnegative robust semantics. For negation-free discrete-time STL formulas with bounded metric intervals (which we have in evaluation), a trajectory satisfies the STL formula if and only if its robust semantics are nonnegative. \textbf{Actuation Limit Satisfaction Rate (ALSR)} measures the percentage of trajectories satisfying the actuation limit at all time\footnote{We treat actuation as within limit if it lies within $[u_{min, i} - \epsilon, u_{max, i} + \epsilon]$ for all entries $i$ with $\epsilon \coloneq 10^{-5}$ to prevent floating-point issues.}. Lastly, \textbf{timing}, in seconds, measures the overall computation time for the closed-loop planning (i.e. the time needed to solve for $x^*$ in problem \ref{prob:formulation}) averaged across the 20 trials. We show the results across all baselines in Table \ref{tab:mars_rov}\footnote{In Case A with limited actuation when evaluating safety-aware-stl-mppi, one of the 20 trials terminated with an infeasible CBF-QP (Assumption 1 fails for that initial condition under the tight actuation bound); the planner reports infeasibility rather than executing an unsafe action. SR and timing are computed over the 19 completed trials.}. All methods achieved an ALSR of $100\%$ across all cases. Importantly, our method demonstrates $100\%$ empirical STL satisfaction rate in the evaluated trials and computation \textbf{efficiency} under \textbf{diverse cost settings and environments}. Vanilla MPPI is efficient but violates safety constraints frequently. Reach-avoid MPPI cannot generalize to complex STL tasks with tight timed constraints. Robustness and penalty-based MPPIs suffer from high computation costs due to the long-horizon rollouts and do not have safety guarantee.


\subsection{Quadcopter Simulation}
We consider a drone planning simulation in an artificial factory environment in Isaac Lab with a Crazyflie quadcopter. The environment consists of seven obstacles including three package tables, one two-level stand and three forklifts. We estimate the obstacles with infinite cylinder bounding volumes, which we denote with $C_i$ where $i \in \{1, \hdots, 7\}$. We first consider a planning stage (with $200 HZ$ planning frequency) where we model the quadcopter with a double integrator dynamics (see case D from Section \ref{subsec:artificial} but extended to $3$ dimensions), where $x_t \coloneq (p_t, v_t)$ with $p_t \coloneq (p^x, p^y, p^z)^T \in \mathbb{R}^3$ and $v_t \in \mathbb{R}^3$. We impose an actuation limit of $(-0.2, -0.2, -0.2)^T \preccurlyeq u_t \preccurlyeq (0.2, 0.2, 0.2)^T$. The goal of the task is to safely navigate near $(0, 0, 1)^T$ without collision with the obstacles nor the ground. Formally, we consider the safety specification $\phi \coloneq \wedge_{i = 1}^7 G_{[0s, 5s]} S(p, C_i) \ge 0 \wedge G_{[0s, 5s]} (p^z - 0.3 ) \ge 0 \wedge F_{[5s, 5s]} \|p - (0, 0, 1)^T\| \le 1$, where $S(p, C)$ denotes the signed distance between $p$ and set $C$. We consider $50$ experimental trials, where in each experiment, we use safety-aware-stl-mppi to plan a trajectory. To allow a smooth trajectory converging towards the goal with minimal actuation expense, a stable height, and a small velocity, we consider a stage cost of $q(x_\tau^{k, (t)}, u^{k, (t)}) \coloneq \|u^{k, (t)}\|_2 + |p_\tau^{k, (t), z} - 1| + \|v_\tau^{k, (t)}\|_2$ and a terminal cost of $E(x_{t + H}^{k, (t)}) \coloneq 0$. See Figure \ref{fig:isaac_planned} for an example of a successfully planned trajectory, denoted by the red waypoints. We show an example of the rollouts during planning in Figure \ref{fig:isaac_rollout}, where the gray waypoints denote the rollouts $x^k$ with $K \coloneq 100$ and the yellow waypoints are the weighted rollout. We record the average time of $15.21$ seconds from planning across the $50$ trials, with all executions satisfying the actuation limit during planning. In the second stage, we track the planned waypoints with position update through Isaac Sim API calls. We record that $100\%$ of the tracked trajectories satisfy the safety specification. \anonymize{See example video of the quadcopter flying based on the above pipeline in \url{placeholder}.}

\section{Conclusion}
In this paper, we propose safety-aware-stl-mppi, a receding-horizon sampling-based planner under STL safety specifications by integrating a time-varying CBF. We conducted a series of case studies with Mars Rover planning and a quadcopter simulation, which demonstrate that our method maintains high safety and efficiency outcomes. The next step will focus on extending the results to multi-agent tasks and dynamics with uncertainty.

\authorswitch{\section{Acknowledgement}
We would like to thank Hardik Parwana for discussing solutions for higher-order CBFs and Sung Woo Choi for discussing use cases and experiments with our algorithm. We would like to thank Paul Lutkus for ideas on visualizing the MPPI rollouts in the quadcopter study and suggestions in presentation. We acknowledge the adaptation of \cite{um-arm-lab_pytorch_mppi_2024} for our implementation of MPPI. This work was partially supported by the National Science Foundation through the following grants: CAREER award (SHF-2048094), IIS-SLES-2417075, and funding by Toyota R\&D through the USC Center for Autonomy and AI.}{}

\bibliographystyle{ieeetr}
\bibliography{main} 

@inproceedings{maler2004monitoring,
  title={Monitoring temporal properties of continuous signals},
  author={Maler, Oded and Nickovic, Dejan},
  booktitle={FTRTFT},
  pages={152--166},
  year={2004},
  organization={Springer}
}

@article{yin2023shield,
  title={Shield model predictive path integral: A computationally efficient robust mpc method using control barrier functions},
  author={Yin, Ji and Dawson, Charles and Fan, Chuchu and Tsiotras, Panagiotis},
  journal={RA-L},
  volume={8},
  number={11},
  pages={7106--7113},
  year={2023},
  publisher={IEEE}
}

@article{lindemann2018control,
  title={Control barrier functions for signal temporal logic tasks},
  author={Lindemann, Lars and Dimarogonas, Dimos V},
  journal={L-CSS},
  volume={3},
  number={1},
  pages={96--101},
  year={2018},
  publisher={IEEE}
}

@article{han2026signal,
  title={Signal Temporal Logic Constrained Motion Control of Electric Vehicles under Slippery Roads Using an MPPI Approach},
  author={Han, Jinheng and Li, Chao and Hou, Xiaohui and Pi, Dawei and Yin, Guodong and Zhao, Jing and Wong, Pak Kin and Wang, Yan},
  journal={IEEE Transactions on Transportation Electrification},
  year={2026},
  publisher={IEEE}
}

@article{marchesini2026sampling,
  title={Sampling-based planning under stl specifications: A forward invariance approach},
  author={Marchesini, Gregorio and Liu, Siyuan and Lindemann, Lars and Dimarogonas, Dimos V},
  journal={IEEE Transactions on Automatic Control},
  year={2026},
  publisher={IEEE}
}

@article{halder2026lexicographic,
  title={Lexicographic Minimum-Violation Motion Planning using Signal Temporal Logic},
  author={Halder, Patrick and Kiltz, Lothar and Homburger, Hannes and Reuter, Johannes and Althoff, Matthias},
  journal={arXiv preprint arXiv:2604.20428},
  year={2026}
}

@article{bouzid2026autonomous,
  title={Autonomous Driving with Priority-Ordered STL Specifications Under Multimodal Uncertainty},
  author={Bouzid, Taha and Qi, Shuhao and Lazar, Mircea and Haesaert, Sofie},
  journal={arXiv preprint arXiv:2606.20336},
  year={2026}
}

@article{zheng2026stl,
  title={STL-SVPIO: Signal Temporal Logic guided Stein Variational Path Integral Optimization},
  author={Zheng, Hongrui and Zang, Zirui and Amine, Ahmad and Vasile, Cristian Ioan and Mangharam, Rahul},
  journal={arXiv preprint arXiv:2603.13333},
  year={2026}
}

@inproceedings{sadraddini2015robust,
  title={Robust temporal logic model predictive control},
  author={Sadraddini, Sadra and Belta, Calin},
  booktitle={Allerton},
  pages={772--779},
  year={2015},
  organization={IEEE}
}

@article{fainekos2009robustness,
  title={Robustness of temporal logic specifications for continuous-time signals},
  author={Fainekos, Georgios E and Pappas, George J},
  journal={Theoretical Computer Science},
  volume={410},
  number={42},
  pages={4262--4291},
  year={2009},
  publisher={Elsevier}
}

@inproceedings{donze2010robust,
  title={Robust satisfaction of temporal logic over real-valued signals},
  author={Donz{\'e}, Alexandre and Maler, Oded},
  booktitle={FORMATS},
  pages={92--106},
  year={2010},
  organization={Springer}
}

@article{williams2018information,
  title={Information-theoretic model predictive control: Theory and applications to autonomous driving},
  author={Williams, Grady and Drews, Paul and Goldfain, Brian and Rehg, James M and Theodorou, Evangelos A},
  journal={T-RO},
  volume={34},
  number={6},
  pages={1603--1622},
  year={2018},
  publisher={IEEE}
}

@inproceedings{tao2022control,
  title={Control barrier function augmentation in sampling-based control algorithm for sample efficiency},
  author={Tao, Chuyuan and Kim, Hunmin and Yoon, Hyungjin and Hovakimyan, Naira and Voulgaris, Petros},
  booktitle={ACC},
  pages={3488--3493},
  year={2022},
  organization={IEEE}
}

@article{parwana2025br,
  title={BR-MPPI: Barrier Rate guided MPPI for Enforcing Multiple Inequality Constraints with Learned Signed Distance Field},
  author={Parwana, Hardik and Kim, Taekyung and Long, Kehan and Hoxha, Bardh and Okamoto, Hideki and Fainekos, Georgios and Panagou, Dimitra},
  journal={arXiv preprint arXiv:2506.07325},
  year={2025}
}

@article{baldini2024don,
  title={Don't Get Stuck: A Deadlock Recovery Approach},
  author={Baldini, Francesca and Tariq, Faizan M and Bae, Sangjae and Isele, David},
  journal={arXiv preprint arXiv:2408.10167},
  year={2024}
}

@article{halder2025trajectory,
  title={Trajectory planning with signal temporal logic costs using deterministic path integral optimization},
  author={Halder, Patrick and Homburger, Hannes and Kiltz, Lothar and Reuter, Johannes and Althoff, Matthias},
  journal={arXiv preprint arXiv:2503.01476},
  year={2025}
}

@inproceedings{raman2014model,
  title={Model predictive control from signal temporal logic specifications: A case study},
  author={Raman, Vasumathi and Maasoumy, Mehdi and Donz{\'e}, Alexandre},
  booktitle={Proc.of the 4th ACM SIGBED International Workshop on Design, Modeling, and Evaluation of Cyber-Physical Systems},
  pages={52--55},
  year={2014}
}

@inproceedings{williams2016aggressive,
  title={Aggressive driving with model predictive path integral control},
  author={Williams, Grady and Drews, Paul and Goldfain, Brian and Rehg, James M and Theodorou, Evangelos A},
  booktitle={ICRA},
  pages={1433--1440},
  year={2016},
  organization={IEEE}
}

@article{homburger2025optimality,
  title={Optimality and suboptimality of MPPI control in stochastic and deterministic settings},
  author={Homburger, Hannes and Messerer, Florian and Diehl, Moritz and Reuter, Johannes},
  journal={L-CSS},
  year={2025},
  publisher={IEEE}
}

@inproceedings{dimos2019learning,
  title={A learning framework for versatile STL controller synthesis},
  author={Dimos, Peter Varnai and Dimarogonas, Dimos V},
  booktitle={CDC},
  pages={4596--4600},
  year={2019},
  organization={IEEE}
}

@article{varnai2020guided,
  title={Guided policy improvement for satisfying STL tasks using funnel adaptation},
  author={Varnai, Peter and Dimarogonas, Dimos V},
  journal={arXiv preprint arXiv:2004.05653},
  year={2020}
}

@inproceedings{varnai2019prescribed,
  title={Prescribed performance control guided policy improvement for satisfying signal temporal logic tasks},
  author={Varnai, Peter and Dimarogonas, Dimos V},
  booktitle={ACC},
  pages={286--291},
  year={2019},
  organization={IEEE}
}

@article{parwana2024model,
  title={Model predictive path integral methods with reach-avoid tasks and control barrier functions},
  author={Parwana, Hardik and Black, Mitchell and Fainekos, Georgios and Hoxha, Bardh and Okamoto, Hideki and Prokhorov, Danil},
  journal={arXiv preprint arXiv:2407.13693},
  year={2024}
}

@inproceedings{vasile2017sampling,
  title={Sampling-based synthesis of maximally-satisfying controllers for temporal logic specifications},
  author={Vasile, Cristian-Ioan and Raman, Vasumathi and Karaman, Sertac},
  booktitle={IROS},
  pages={3840--3847},
  year={2017},
  organization={IEEE}
}

@inproceedings{pnueli1977temporal,
  title={The temporal logic of programs},
  author={Pnueli, Amir},
  booktitle={sfcs 1977},
  pages={46--57},
  year={1977},
  organization={ieee}
}

@article{fainekos2009temporal,
  title={Temporal logic motion planning for dynamic robots},
  author={Fainekos, Georgios E and Girard, Antoine and Kress-Gazit, Hadas and Pappas, George J},
  journal={Automatica},
  volume={45},
  number={2},
  pages={343--352},
  year={2009},
  publisher={Elsevier}
}

@article{ding2014optimal,
  title={Optimal control of Markov decision processes with linear temporal logic constraints},
  author={Ding, Xuchu and Smith, Stephen L and Belta, Calin and Rus, Daniela},
  journal={TAC},
  volume={59},
  number={5},
  pages={1244--1257},
  year={2014},
  publisher={IEEE}
}

@article{kantaros2020stylus,
  title={Stylus*: A temporal logic optimal control synthesis algorithm for large-scale multi-robot systems},
  author={Kantaros, Yiannis and Zavlanos, Michael M},
  journal={IJRR},
  volume={39},
  number={7},
  pages={812--836},
  year={2020},
  publisher={SAGE Publications Sage UK: London, England}
}

@article{zhong2022guided,
  title={Guided conditional diffusion for controllable traffic simulation},
  author={Zhong, Ziyuan and Rempe, Davis and Xu, Danfei and Chen, Yuxiao and Veer, Sushant and Che, Tong and Ray, Baishakhi and Pavone, Marco},
  journal={arXiv preprint arXiv:2210.17366},
  year={2022}
}

@article{meng2023signal,
  title={Signal temporal logic neural predictive control},
  author={Meng, Yue and Fan, Chuchu},
  journal={RA-L},
  volume={8},
  number={11},
  pages={7719--7726},
  year={2023},
  publisher={IEEE}
}

@article{paszke2019pytorch,
  title={Pytorch: An imperative style, high-performance deep learning library},
  author={Paszke, Adam and Gross, Sam and Massa, Francisco and Lerer, Adam and Bradbury, James and Chanan, Gregory and Killeen, Trevor and Lin, Zeming and Gimelshein, Natalia and Antiga, Luca and others},
  journal={NeurIPS},
  volume={32},
  year={2019}
}

@article{makoviychuk2021isaac,
  title={Isaac gym: High performance gpu-based physics simulation for robot learning},
  author={Makoviychuk, Viktor and Wawrzyniak, Lukasz and Guo, Yunrong and Lu, Michelle and Storey, Kier and Macklin, Miles and Hoeller, David and Rudin, Nikita and Allshire, Arthur and Handa, Ankur and others},
  journal={arXiv preprint arXiv:2108.10470},
  year={2021}
}

@article{mittal2023orbit,
   author={Mittal, Mayank and Yu, Calvin and Yu, Qinxi and Liu, Jingzhou and Rudin, Nikita and Hoeller, David and Yuan, Jia Lin and Singh, Ritvik and Guo, Yunrong and Mazhar, Hammad and Mandlekar, Ajay and Babich, Buck and State, Gavriel and Hutter, Marco and Garg, Animesh},
   journal={RA-L},
   title={Orbit: A Unified Simulation Framework for Interactive Robot Learning Environments},
   year={2023},
   volume={8},
   number={6},
   pages={3740-3747},
   doi={10.1109/LRA.2023.3270034}
}

@misc{um-arm-lab_pytorch_mppi_2024,
  title        = {pytorch\_mppi},
  author       = {Autonomous Robotic Manipulation Lab, University of Michigan},
  year         = {2024},
  howpublished = {\url{https://github.com/UM-ARM-Lab/pytorch_mppi}},
  note         = {Accessed: 2025-08-17}
}

@book{lindemann2025formal,
  title={Formal Methods for Multi-Agent Feedback Control Systems},
  author={Lindemann, Lars and Dimarogonas, Dimos V},
  year={2025},
  publisher={MIT Press}
}

@article{zhao2026logic,
  title={Logic-VLA: A Temporal Logic Conditioned Vision-Language-Action Model},
  author={Wang, Celina Shiyu and Zhao, Yiqi and Ye, Junjie and Wang, Yue and Deshmukh, Jyotirmoy V},
  journal={arXiv preprint arXiv:2608.20556},
  year={2026}
}

@incollection{bartocci2018specification,
  title={Specification-based monitoring of cyber-physical systems: a survey on theory, tools and applications},
  author={Bartocci, Ezio and Deshmukh, Jyotirmoy and Donz{\'e}, Alexandre and Fainekos, Georgios and Maler, Oded and Ni{\v{c}}kovi{\'c}, Dejan and Sankaranarayanan, Sriram},
  booktitle={Lectures on Runtime Verification: Introductory and Advanced Topics},
  pages={135--175},
  year={2018},
  publisher={Springer}
}

@inproceedings{hekmatnejad2019encoding,
  title={Encoding and monitoring responsibility sensitive safety rules for automated vehicles in signal temporal logic},
  author={Hekmatnejad, Mohammad and Yaghoubi, Shakiba and Dokhanchi, Adel and Amor, Heni Ben and Shrivastava, Aviral and Karam, Lina and Fainekos, Georgios},
  booktitle={MEMOCODE},
  pages={1--11},
  year={2019}
}

@inproceedings{kordabad2024control,
  title={Control barrier functions for stochastic systems under signal temporal logic tasks},
  author={Kordabad, Arash Bahari and Charitidou, Maria and Dimarogonas, Dimos V and Soudjani, Sadegh},
  booktitle={ECC},
  pages={3213--3219},
  year={2024},
  organization={IEEE}
}
\appendixswitch{
\section{Appendix}

\subsection{Proof for Theorem \ref{thm:forward}}
\label{proof:forward}
\begin{proof}
    The theorem can be proven by induction (similar to the proof of Property III.1 in \cite{yin2023shield}): We assumed $x_0 \in \mathcal{C}_0$. At step $t$, knowing $\alpha \in (0, 1]$, we assume $x_t \in \mathcal{C}_t$ and $u_t \in S(x_t, t)$ and thus $h(f(x_t, u_t), t + 1) \ge h(x_t, t) -\alpha h(x_t, t) \ge 0$, which implies $f(x_t, u_t) \in \mathcal{C}_{t + 1}$.
\end{proof}

\subsection{Proof for Theorem \ref{thm:safety_cbf}}
\label{proof:safety_cbf}
\begin{proof}
        We are given the condition $x_0 
        \in C_0$. From Theorem \ref{thm:forward}, we know $x_{t_i} \in \mathcal{C}_{t_i} \implies x_t \in \mathcal{C}_t$ for $t \in \{t_i, \hdots, t_{i + 1}\}$. Since the exponential function is strictly positive and logarithm is monotonic, we have $h(x, t_{i + 1}) \le h^+(x, t_{i + 1})$ for all $x \in D$, where $h^+(x,t_{i+1}) := \widetilde{\text{min}}(\mathcal{A}_{t_{i+1}+1}(x,t))$ denotes the deletion. The update is forward invariant ($h(x, t_{i + 1}) \ge 0 \implies h^+(x, t_{i + 1}) \ge 0$). Since $\widetilde{\text{min}}$ lower-bounds each argument, $h^\phi(x_t,t)\ge 0$ implies $h_i(x_t,t)\ge 0$ for all $t\in\{t_i,\ldots,t_{i+1}\}$, including the endpoint $t_{i+1}$; by the conditions of Table I this implies satisfaction of each $\phi_i$, and hence $(x,0)\models\phi$ by induction on the structure of the formula: For an atomic predicate $\pi^\mu$, the implication follows directly from the first row of Table \ref{tab:stl_encodings}. For $G_{[a, b]}$, Table \ref{tab:stl_encodings} gives the predicate at every $t \in \sigma(\tau_0 + a, \tau_0 + b)$. For $F_{[a, b]}$, Table \ref{tab:stl_encodings} gives a witness $t' \in \sigma(\tau_0 + a, \tau_0 + b)$. For $U_{[a. b]}$, Table \ref{tab:stl_encodings} gives a witness for the right operand and the left operand up to that witness. Conjunction follows because the smooth minimum being nonnegative implies every constituent barrier is nonnegative.
    \end{proof}

\subsection{Derivation of Equation \eqref{eq:qp_closed}}
\label{appendix:qp} For simple notation, let $\hat{v}^s \coloneq \hat{v}^{s, (t)}_\tau, v^+ \coloneq v_\tau^{+, (t)}, a \coloneq a_{\textup{cbf}}(x_\tau^{(t)}, \tau), \text{ and } b \coloneq b_{\textup{cbf}}(x_\tau^{(t)}, \tau)$. We assumed $a \neq 0$ in Section \ref{sec:alg}. Assume $\mathcal{U} \coloneq \mathbb{R}^m$. Then, we can write \eqref{eq:qp} as:
\begin{align*}
    \hat{v}^{s}_\tau =& \arg\min_{v \in \mathbb{R}^m}\{\frac{1}{2}\|v - v^+\|^2 \mid a^Tv \le b\}. \nonumber
\end{align*}

Suppose $a^Tv^+ \le b$, it is intuitive to have $\hat{v}^s = v^+$ because this corresponds to an already safe action, and the optimal objective cost is $0$ and any $\hat{v}^s \neq v^+$ will result in a straight positive cost.

On the other hand, suppose $a^Tv^+ > b$. We have the Lagrangian $\mathcal{L}(v, \mu) = \frac{1}{2}\|v - v^+\|^2 + \mu(a^Tv - b)$. By the KKT conditions, the stationarity states $\hat{v}^s = v^+ - \mu a$ and the complementary slackness states $\mu(a^T\hat{v}^s - b) = 0$. Since $a^Tv^+ > b$, the optimum must lie on the boundary and thus $a^T\hat{v}^s = b$. By substitution, we have $a^T(v^+ - \mu a) = b$ and thus it must hold that $\mu \coloneq (a^Tv^+ - b) / \|a\|_2^2$. Therefore, $\hat{v}^s = v^+ - a[(a^Tv^+ - b)/ \|a\|_2^2]$. Note that here the KKT conditions are sufficient and necessary since \eqref{eq:qp} is strictly convex with affine feasible region containing nonempty interior.

Combining the above two cases, one can derive the closed form solution $\hat{v}^s = v^+ - a[\max(0, a^Tv^+ - b)/ \|a\|_2^2]$.

\subsection{Theorem \ref{thm:tighten_local}}
\label{thm:tighten_local}
\begin{theorem}
\label{thm:tighten_local}
    Suppose $\nabla h^\phi(\cdot, \tau + 1)$ is $L_{x, \tau}$-lipschitz on a set $K_{x}$ containing every segment $\{x + s(f(x, v) - x) : s \in [0, 1]\}$ where $v \in \mathcal{U}$ (i.e., $\|\nabla_x h^\phi(x_1, \tau + 1) - \nabla_xh^\phi(x_2, \tau + 1)\| \le L_{x, \tau}\|x_1 - x_2\|, \forall x_1, x_2 \in K_x$). Assume also $\Delta_{x} \coloneq\sup_{v \in \mathcal{U}} \|f(x, v) - x\| < \infty$. Then \eqref{eq:cons_replaced} defines a subset of the feasibility region to \eqref{eq:cbf_opt} if the right-hand side of \eqref{eq:cons_replaced} is incremented\footnote{The norms in Theorem \ref{thm:tighten_local} are Euclidean norms.} by $\epsilon_{x, \tau} \coloneq \frac{L_{x, \tau}}{2}\Delta_{x}^2$.
    
    \begin{proof}
    Let us denote $P(\cdot) \coloneq h^\phi(\cdot, \tau + 1)$ and $d_v = f(x, v) - x$. We can show the local version of the descent lemma. Per assumption, $\nabla P$ is $L_{x, \tau}$-lipschitz on $\{x + sd_v: s \in [0, 1], \forall v \in \mathcal{U}\}$. Suppose $g_v(s) = P(x + sd_v)$, then
    \begin{align*}
        &P(f(x, v)) - P(x) \\= &\int_0^1 \nabla P(x + sd_v)^Td_vds \\= &\nabla P(x)^Td_v + \int_0^1( \nabla P(x + sd_v) - \nabla P(x))^Td_vds.
    \end{align*}
    We then have by Cauchy-Schwarz and our lipschitz assumption that
    \begin{align*}
        &(\nabla P(x + sd_v) - \nabla P(x))^Td_v \\\ge &- \|\nabla P(x + sd_v) - \nabla P(x)\|\|d_v\| \\\ge &- L_{x, \tau}s\|d_v\|^2.
    \end{align*}
    Therefore, we have
    \begin{align*}
        P(f(x, v)) - P(x) \ge \nabla P(x)^Td_v - \frac{L_{x, \tau}}{2}\|d_v\|^2
    \end{align*}
    which holds for all $v \in \mathcal{U}$. By assumption, we have $\|d_v\| = \|f(x, v) - x\| \le \Delta_x, \forall v \in \mathcal{U}$. We thus have
    \begin{align*}
        P(x) +  \nabla P(x)^Td_v - \frac{L_{x, \tau}}{2}\Delta_x^2 \le P(f(x, v))
    \end{align*}
    Hence, if
    \begin{align*}
        &P(x) +  \nabla P(x)^Td_v - \frac{L_{x, \tau}}{2}\Delta_x^2
        \\\coloneq & h^\phi(x, \tau + 1) + \nabla_xh^\phi(x, \tau + 1)[(f_0(x) - x) \\&+ g(x)v] - \frac{L_{x, \tau}}{2}\Delta_x^2 \ge (1 -\alpha) h^\phi(x, \tau),
    \end{align*}
    then it holds that
    \begin{align*}
        h^\phi(f(x, v), \tau + 1) \ge (1 -\alpha) h^\phi(x, \tau).
    \end{align*}
\end{proof}
\end{theorem}

By Theorem \ref{thm:tighten_local}, we can derive a robustified algorithm (which we call r-safety-aware-stl-mppi), where we replace $b_{\text{cbf}}(x_\tau^{(t)}, \tau)$ in \eqref{eq:qp} and \eqref{eq:qp_closed} with $b^r_{\text{cbf}}(x_\tau^{(t)}, \tau) \coloneq b_{\text{cbf}}(x_\tau^{(t)}, \tau) - \epsilon_{x, \tau}$ in Algorithm \ref{alg:algorithm}.

Computing $L_{x, \tau}$ is generally nontrivial, and one can use a heuristic sampling-based estimation
\begin{align}
\label{eq:lipschitz_estimation}
    L_{x, \tau} \approx \max_{x_i, x_j \in K_x, i \neq j} \frac{\|\nabla_x h^\phi(x_i, \tau + 1) - \nabla_xh^\phi(x_j, \tau + 1)\|}{\|x_i - x_j\|}.
\end{align}
However, for control affine dynamics $f(x, u) = f_0(x) + g(x)u$ and boxed input constraints $\mathcal{U}$, it is easy to show that $\Delta_x = \max_{v \in \text{vert}(\mathcal{U})} \|f_0(x) - x + g(x)v\|$ where $\text{vert}(\cdot)$ denote the set of box vertices. Since $\mathcal{U}$ is a convex hull of its vertices $w_1, \hdots, w_k$, for any $v \in U$, we can express $v = \sum_{i = 1}^k \lambda_iw_i$ where $\lambda_i \ge 0$ are coefficients such that $\sum_i\lambda_i = 1$. Suppose $p(v) = \|(f_0(x) - x) + g(x)v\|_2$. Since $p$ is convex, we have
\begin{align*}
    p(v) = p(\sum_{i = 1}^k \lambda_iw_i) \le \sum_{i = 1}^k\lambda_i p(w_i) \le \max_i p(w_i)
\end{align*}
and we can conclude that \textbf{in this setting}
\begin{align*}
\Delta_x = \max_{v \in \text{vert}(\mathcal{U})} \|f_0(x) - x + g(x)v\|_2.
\end{align*}
We remark on the conservatism of this robustification. Also, compared to safety-aware-stl-mppi, the robustified method is more time consuming due to the computation of $\epsilon_{x, \tau}$.}{}
\end{document}